\documentclass{article}

\usepackage[numbers,sort&compress]{natbib}

\usepackage[preprint]{neurips_2025}

\usepackage[utf8]{inputenc} % allow utf-8 input
\usepackage[T1]{fontenc}    % use 8-bit T1 fonts
\usepackage{hyperref}       % hyperlinks
\usepackage{url}            % simple URL typesetting
\usepackage{booktabs}       % professional-quality tables
\usepackage{amsfonts}       % blackboard math symbols
\usepackage{nicefrac}       % compact symbols for 1/2, etc.
\usepackage{microtype}      % microtypography
\usepackage{xcolor}         % colors
\usepackage{amsmath}
\usepackage{subcaption}
\usepackage{graphicx} 
\usepackage{appendix}
\usepackage{amsthm}

\newtheorem{corollary}{Corollary}
\newtheorem{proposition}{Proposition}
\newtheorem{theorem}{Theorem}
\newtheorem{definition}{Definition}
\newtheorem{remark}{Remark}

\title {A Principled Bayesian Framework for Training Binary and Spiking Neural Networks}

 \author{%
   James A. Walker \\
   Monash University\\
   Melbourne, Australia \\
   \texttt{james.walker2@monash.edu} \\
   \And
   Moein Khajehnejad \\
   Monash University\\
   Melbourne, Australia \\
   \texttt{moein.khajehnejad@monash.edu}\\
   \AND
    Adeel Razi \\
    Monash University\\
    Melbourne, Australia \\
    \texttt{adeel.razi@monash.edu} \\
 }

\begin{document}

\maketitle

\begin{abstract}

We propose a Bayesian framework for training binary and spiking neural networks that achieves state-of-the-art performance without normalisation layers. Unlike commonly used surrogate gradient methods -- often heuristic and sensitive to hyperparameter choices -- our approach is grounded in a probabilistic model of noisy binary networks, enabling fully end-to-end gradient-based optimisation. We introduce importance-weighted straight-through (IW-ST) estimators, a unified class generalising straight-through and relaxation-based estimators. We characterise the bias-variance trade-off in this family and derive a bias-minimising objective implemented via an auxiliary loss. Building on this, we introduce Spiking Bayesian Neural Networks (SBNNs), a variational inference framework that uses posterior noise to train Binary and Spiking Neural Networks with IW-ST. This Bayesian approach minimises gradient bias, regularises parameters, and introduces dropout-like noise. By linking low-bias conditions, vanishing gradients, and the KL term, we enable training of deep residual networks without normalisation. Experiments on CIFAR-10, DVS Gesture, and SHD show our method matches or exceeds existing approaches without normalisation or hand-tuned gradients.

\end{abstract}

%\maketitle
\section{Introduction}

Binary neural networks (BNNs) and spiking neural networks (SNNs) offer energy-efficient, biologically inspired alternatives to conventional deep learning. BNNs use binary activations, greatly reducing memory and computation, ideal for resource-constrained settings \cite{courbariaux2016binarizedneuralnetworkstraining, rastegari2016xnornetimagenetclassificationusing}. SNNs emulate the sparse, event-driven behavior of biological neurons, enabling low-power, real-time processing of temporal data \cite{Tavanaei_2019}.

Training these networks is challenging due to their non-differentiable, discontinuous activations, making standard gradient-based methods inapplicable. Most approaches use \textit{surrogate gradient} (SG) methods \cite{Neftci2019}, replacing non-differentiable functions with smooth approximations during backpropagation. While effective in practice, SG lacks theoretical grounding, requires manual tuning, and in deterministic binary networks, gradients are zero almost everywhere—violating standard gradient descent assumptions \cite{yin2019understanding}.
%This is largely due to the fact that in these deterministic networks the true gradient for a given training example is 0 almost everywhere, meaning we cannot appeal directly to the principle of gradient descent, and must instead make strong assumptions about the data distribution and network structure \cite{yin2019understanding}.
A principled alternative is to introduce stochasticity by sampling neuron outputs from binary distributions (e.g., Bernoulli), making the expected loss differentiable and enabling gradient-based optimisation. Building on this, we revisit gradient estimation in noisy binary networks and unify straight-through (ST) and continuous relaxation methods into a novel class: \textit{importance-weighted straight-through} (IW-ST) estimators. We analyse their bias–variance trade-off, showing how classical ST implicitly balances it, and derive an objective that reduces ST bias even in arbitrarily deep networks
%We discuss the two main approaches that enable backpropagation in such networks, the straight-through (ST) estimator and continuous-relaxation based methods, and show how a series of improvements to the latter motivates a new class of estimators we call importance-weighted SG estimators (IW-ST), of which the original ST estimator is an example. We then analyse the bias-variance trade-off of this family of estimators for the general case of deep networks, and derive conditions which ensures both the bias and variance are small, which can be tied to the vanishing gradient problem.

Building on these insights, we propose \textit{Spiking Bayesian Neural Networks} (SBNNs), a variational inference framework for training BNNs and SNNs that naturally incorporates noise and regularisation, extending Bayesian Neural Networks \cite{blundell2015weight} to spiking models. Using the local reparameterisation trick \cite{kingma2015variational}, we convert weight uncertainty into neuron-level noise, enabling end-to-end training with IW-ST estimators.

%This approach assigns uncertainty to the weights, which via the local reparameterisation trick \cite{kingma2015variational}, acts as the source of noise enabling the ST estimator to be used. 
Crucially, the KL divergence term in the variational objective encourages a noise level that both prevents vanishing gradients and minimises estimator bias, eliminating the need for normalisation layers, weight decay, or dropout.
%We can then use variational inference as an end-to-end optimisation scheme for SNNs which, in contrast to the SG method, is theoretically grounded,  does not require heuristic SG functions to be chosen, and naturally incorporates noise and weight regularisation. Additionally, we show that the KL divergence term intrinsic to the variational inference setting can actively reduce the bias and variance of the ST estimator throughout training, and help mitigate vanishing gradients, ultimately enabling normalisation layers to be omitted entirely from common network architectures such as residual networks. 

We evaluate our approach on CIFAR-10 \cite{krizhevsky2009learning}, SHD \cite{cramer2020heidelberg}, and DVS-128 Gesture \cite{amir2017low}, consistently matching or surpassing SG-based methods without relying on normalisation or hand-tuned surrogates. Results highlight Bayesian noise as a powerful, principled tool for training discrete and spiking networks.
%that this framework is highly effective for training BNNs and SNNs without normalisation, achieving competitive performance with state of the art methods which require normalisation and hyperparameter optimisation.

%\vspace{-0.1in}
%\subsubsection*{Contributions and related work}
\textbf{Prior Work.}
In BNNs, the ST estimator is viewed as a linear approximation of a finite-difference gradient \cite{tokui17a}, with bias–variance analyses limited to shallow networks \cite{shekhovtsov2021}. Continuous relaxation methods share similar trade-offs \cite{shekhovtsov2021biasvariancetradeoff} and benefit from Rao-Blackwellisation to reduce variance \cite{paulus2020}, but their connection to ST remains unclear, and analyses don’t generalise to deep networks. In SNNs, SG methods are common \cite{Neftci2019} but depend on hand-crafted surrogates, normalisation, and regularisation. While recent work links noisy gradients to SG \cite{Gygax2024} and applies variational inference to binary weights \cite{peters2018, meng2020}, a unified, principled training framework for BNNs and SNNs is still lacking.

\textbf{Our Contributions.}
We introduce a unified class of importance-weighted straight-through (IW-ST) estimators that generalises classical ST and continuous relaxation methods via analytical Rao-Blackwellisation. Extending bias–variance analysis to deep networks, we provide conditions for low-bias, low-variance gradient estimation. Building on this, we propose Spiking Bayesian Neural Networks (SBNNs), a variational inference framework enabling end-to-end training of BNNs and SNNs without normalisation, weight decay, or dropout. The KL term in the ELBO mitigates vanishing gradients, reduces bias, and regularises the model. This approach extends Bayesian binary weight methods to spiking networks and consistently outperforms SG baselines across benchmarks.

\textbf{Preliminaries and Notations.}
We consider two classes of binary-valued neural models: feedforward binary neural networks (BNNs) and recurrent spiking neural networks (SNNs). Both are built from binary threshold units, but SNNs include temporal dynamics inspired by biological neurons.

% \begin{definition}[Binary Neural Network (BNN)]
A BNN consists of layers of binary neurons, where each neuron outputs either 0 or 1 based on a thresholded preactivation. For a fully connected layer:
\begin{align}
o_i^{(l)} = H(h_i^{(l)}), \quad
h_i^{(l)} = \sum_j w_{ij}^{(l)} o_j^{(l-1)}
\end{align}
where \( h_i^{(l)} \) is the preactivation of neuron \( i \), \( o_i^{(l)} \in \{0,1\} \) is its binary output, and \( w_{ij}^{(l)} \in \mathbb{R} \) denotes the weight from neuron \( j \) in layer \( l-1 \) to neuron $i$ in layer $l$. The activation function \(H(\cdot)\) is the Heaviside step function.
% \end{definition}% \begin{definition}[Spiking Neural Network (SNN)]
]
An SNN augments a BNN with temporal dynamics by introducing recurrent membrane potentials and spike outputs over discrete time steps:
\begin{align}
o_{i,t}^{(l)} &= H(h_{i,t}^{(l)} - \theta) \\
h_{i,t}^{(l)} &= \beta h_{i,t-1}^{(l)} + \sum_j w_{ij}^{(l)} o_{j,t}^{(l-1)} + \sum_{k \ne i} r_{ik}^{(l)} o_{k,t-1}^{(l)} - \theta o_{i,t-1}^{(l)}
\end{align}
where \( \theta \in \mathbb{R} \) is the firing threshold, \( \beta \in [0,1] \) is the membrane decay factor, and \( r_{ik}^{(l)} \) denotes recurrent connections. This is a discrete approximation of the leaky-integrate-and-fire (LIF) model \cite{gerstner2002spiking}. % \end{definition}% \begin{remark}
SNNs are a recurrent generalisation of BNNs and share architectural similarities with gated recurrent units such as LSTMs \cite{hochreiter1997long} and GRUs \cite{cho2014learning}.
% \end{remark}
% \begin{definition}[Supervised Learning Objective]
We assume a supervised setup with loss computed at the final time step \( T \) and output layer \( L \):
\[
\mathcal{L} = \mathcal{L}(o_T^{(L)}, \text{target}).
\]
% \end{definition}

Derivations are presented for BNNs; generalisation to SNNs follows and is discussed where relevant.

% \subsubsection*{Surrogate gradient method}
% The gradient with respect to a given weight $w_{ij}^{(l)}$ in a BNN can be written as:

% \begin{equation*}
% \frac{d\mathcal{L}}{d w_{ij}^{(l)}} 
% = \frac{d\mathcal{L}}{d o_i^{(l)}} \cdot \frac{d o_i^{(l)}}{d h_i^{(l)}} \cdot \frac{d h_i^{(l)}}{d w_{ij}^{(l)}} 
% = \frac{d\mathcal{L}}{d o_i^{(l)}} \cdot H'(h_i^{(l)}) \cdot o_j^{(l-1)}
% \end{equation*}

% Where we can calculate the first term by backpropagation:
% \begin{equation*}
% \frac{d\mathcal{L}}{d o_i^{(l)}} 
% = \sum_k \frac{d\mathcal{L}}{d o_k^{(l+1)}} \cdot \frac{d o_k^{(l+1)}}{d h_k^{(l+1)}} \cdot \frac{d h_k^{(l+1)}}{d o_i^{(l)}} 
% = \sum_k \frac{d\mathcal{L}}{d o_k^{(l+1)}} \cdot H'(h_k^{(l+1)}) \cdot w_{ki}^{(l+1)}
% \end{equation*}

% The derivative of the Heaviside, $H'()$ is $0$ almost everywhere, meaning that the gradient with respect to any given weight is 0 almost everywhere. To address this issue, the  SG method uses a smooth approximation to the Heaviside derivative such as a sigmoid, setting $H'()=S'()$, where $S()$ is the surrogate function (such as a sigmoid). 

% \begin{remark}[Practical Use]

% \end{remark}

\section{Gradients in Noisy BNNs and SNNs}

Despite its heuristic nature, the SG method is widely used in spiking neural networks and BNNs for its empirical effectiveness (see Appendix A for a brief review). It enables efficient training with standard gradient descent optimisers, however, it lacks a rigorous theoretical basis. Here, we formulate gradient estimation in stochastic binary networks, setting the stage for a unified treatment of surrogate, straight-through, and continuous relaxation-based estimators. Let \( o_i^{(l)} \sim \text{Bernoulli}(F(h_i^{(l)})) \), where \( F() \) is a smooth, increasing function (e.g., a sigmoid). In this probabilistic setting, the expected gradient of the loss with respect to a synaptic weight is:
\begin{align}
\frac{d}{d w_{ij}^{(l)}} \mathbb{E}_{o^{(1:L)}}[\mathcal{L}]
= \mathbb{E} \left[ (\mathcal{L}_1 - \mathcal{L}_0) \cdot F'(h_i^{(l)}) \cdot o_j^{(l-1)} \right],
\end{align}
where \(\mathcal{L}_k=\mathcal{L}(o_i^{(l)}=k):=E[\mathcal{L}|o_i^{(l)}=k,h_i^{(l)}=h]\).  This gradient comprises a presynaptic term \(o_j^{(l-1)}\), a post-synaptic term \(F'(h_i^{(l)})\), and a global learning signal \(\mathcal{L}_1 - \mathcal{L}_0\), which we refer to as the finite difference learning signal. 

Estimating this finite difference is challenging. REINFORCE \cite{williams1992simple} provides an unbiased estimator using an output-loss covariance but suffers from high variance. RAM \cite{tokui17a} reduces variance through explicit marginalisation, requiring two forward passes per neuron. ARM \cite{yin2019arm} achieves lower variance using paired noise samples, while still needing two forward passes per layer.

A more scalable solution is to approximate the finite difference via backpropagation. Backpropagation is efficient and generally yields low-variance estimates by exploiting network structure. However, since the finite difference term is inherently non-linear $\mathcal{L}_1-\mathcal{L}_0= \int_0^1 \frac{d\mathcal{L}}{do} do$, backpropagation can only provide a linear approximation.

Two main families of approximations address this: (i) the ST estimator, which replaces the Heaviside’s zero gradient with a nonzero local derivative, and (ii) continuous relaxation methods (Gumbel-Softmax), which reparameterise the binary variable and use a differentiable surrogate. We show that both approaches fall under a general class of importance-weighted ST estimators (IW-ST($p$)) and derive conditions under which their bias and variance can be simultaneously minimised.

\textbf{Straight-through estimator as a linear approximation.}
The ST estimator was originally introduced as a heuristic technique to handle non-differentiable activations in binary networks. However, it was later shown that ST estimator can be theoretically justified as a linear approximation of the finite-difference gradient in stochastic binary networks \cite{shekhovtsov2021}. Replacing the finite-difference term with a linear approximation at the observed output \( o^{(l)} \), we obtain:
\begin{align*}
\mathbb{E}\left[\mathcal{L}(o_i^{(l)} = 1) - \mathcal{L}(o_i^{(l)} = 0)\right] 
\approx \sum_k w^{(l+1)}_{ki} \, F'(h_k^{(l+1)}) \cdot 
\mathbb{E}\left[\mathcal{L}(o_k^{(l+1)} = 1) - \mathcal{L}(o_k^{(l+1)} = 0) \mid o^{(l)} \right].
\end{align*}
This shows that the finite-difference learning signal at layer \( l \) can be recursively approximated by a weighted combination of differences at layer \( l+1 \), scaled by local sensitivity \( F'(h_k^{(l+1)}) \) and connection weights \( w_{ki} \). This approximation makes backpropagation feasible in binary networks by recursively propagating linearised difference terms across layers. However, the quality of approximation degrades in deeper networks due to compounding bias, which remains an open problem.

\begin{proposition}[Equivalence to surrogate gradient]
The recursive form of the ST estimator gradient matches the recursive SG expression:
\begin{align}
\frac{d\mathcal{L}}{d o_{i}^{(l)}} \approx 
\sum_k w^{(l+1)}_{ki} \, F'(h_k^{(l+1)}) \frac{d\mathcal{L}}{do_k^{(l+1)}}
\end{align}
under the identification \( S'(h_i^{(l)}) \approx F'(h_i^{(l)}) \) and \( \mathbb{E}\left[\mathcal{L}(o_i^{(l)} = 1) - \mathcal{L}(o_i^{(l)} = 0)\right] =\frac{d\mathcal{L}}{do_i^{(l)}} \), where \( S(\cdot) \) is the surrogate function and \( F(\cdot) \) is the spiking probability.
\end{proposition}

This reveals that the SG method in deterministic BNNs can be interpreted as estimating the true gradient in a corresponding noisy probabilistic model. A similar observation was made in \cite{Gygax2024}, though in that work the connection is made to the smoothed stochastic derivatives approach from \cite{arya2023automaticdifferentiationprogramsdiscrete}.

\textbf{Continuous relaxations via reparameterisation.}
To enable backpropagation through discrete stochastic units, we start from a continuous-relaxation based approach. We perform the following reparameterisation:
\begin{align}
o_i^{(l)} = H\left(F(h_i^{(l)}) - 1 + u_i\right), \quad u_i \sim \text{Uniform}[0,1].
\end{align}
Since \( H(\cdot) \) is non-differentiable, we introduce a continuous relaxation using a smooth surrogate:
\begin{align}
s_i^{(l)} := S_k(F(h_i^{(l)}) - 1 + u_i) = \frac{1}{1 + e^{- \frac{F(h_i^{(l)}) - 1 + u_i}{k}}},
\end{align}
where \( S_k(\cdot) \) is a temperature-controlled sigmoid. As \( k \to 0 \), \( S_k(x) \to H(x) \). This relaxation allows us to move the gradient inside the expectation, yielding a reparameterised estimator.

\begin{align}
\frac{d\mathcal{L}^s}{d w_{ij}^{(l)}} = 
\frac{d\mathcal{L}^s}{d s_i^{(l)}} \cdot S_k'\left(F(h_i^{(l)}) - 1 + u_i\right) \cdot F'(h_i^{(l)}) \cdot s_j^{(l-1)},
\end{align}
where \( \mathcal{L}^s \) be the loss computed on the relaxed outputs. 

As \( k \to 0 \), this estimator becomes unbiased; however, the variance increases sharply, often destabilizing training \cite{shekhovtsov2021biasvariancetradeoff}.  To combat this issue the GS-ST was proposed \cite{jang2017categorical}, which uses the binary output but backpropagates through the relaxed surrogate:
\begin{align}
\frac{d\mathcal{L}}{d w_{ij}^{(l)}} \approx 
\left. \frac{d\mathcal{L}}{d s_i^{(l)}} \right|_{s_i^{(l)} = o_i^{(l)}}
\cdot S_k'\left(F(h_i^{(l)}) - 1 + u_i\right)
\cdot F'(h_i^{(l)}) 
\cdot o_j^{(l-1)}.
\end{align}
This estimator is biased for all \(k\), but typically exhibits a better bias-variance trade-off than pure continuous-relaxation.

To further reduce the variance of this estimator, Rao-Blackwellisation can be applied as in \cite{paulus2020}:
\begin{align}
\mathbb{E}_{u \mid o^{(1:L)}}\left[
\frac{d\mathcal{L}}{d o_i^{(l)}} \cdot 
S_k'\left(F(h_i^{(l)}) - 1 + u_i\right)
\right]  o_j^{(l-1)}
= \frac{d\mathcal{L}}{d o_i^{(l)}} \cdot \mathbb{E}_u \left[S_k'\left(F(h_i^{(l)}) - 1 + u_i\right)\right]o_j^{(l-1)}.
\end{align}

Under the Gumbel reparameterization, this conditional expectation is intractable, and Monte Carlo approximations are typically required.

\textbf{Analytical Gumbel-Rao estimator and IW-ST($p$).}
Our choice of \textit{uniform} reparameterization (rather than Gumbel/Logistic), enables us to evaluate the conditional expectation \textit{analytically}, leading to a closed-form estimator.

\begin{definition}[Analytical Gumbel-Rao estimator (AGR)]
The AGR estimator is given by:
\begin{align}
\frac{d\mathcal{L}}{d w_{ij}^{(l)}} =
\left( \frac{S(F(h_i^{(l)})) - S(0)}{F(h_i^{(l)}))} \right)^{o_i^{(l)}} 
\left( \frac{S(0) - S(F(h_i^{(l)})) - 1)}{1 - F(h_i^{(l)}))} \right)^{1 - {o_i^{(l)}} } 
\cdot \frac{d\mathcal{L}}{d o_i^{(l)}} 
\cdot F'(h_i^{(l)}) 
\cdot o_j^{(l-1)}.
\end{align}
\end{definition}

The AGR estimator is a Rao-Blackwellised version of GS-ST and yields lower variance without changing the bias. The first three terms of the AGR expression can be rewritten as:
\begin{align}
(w_0 + w_1) \cdot \mathbb{E}_{o \sim \text{Bern}(\frac{w_1}{w_0 + w_1})} \left[ \frac{d\mathcal{L}}{d o} \right],
\end{align}
where  $w_1 = S(F(h)) - S(0)$, and $w_0 = S(0) - S(F(h) - 1)$. This shows that AGR performs a trapezoidal integration over the derivative \(\frac{d\mathcal{L}}{d o}\), where the sampling distribution is governed by the surrogate's smoothness and the firing probability \( F(h) \). Note that we drop the sub and super-scripts for clarity. The full derivation of this estimator and further discussion is provided in Appendix B.

\begin{remark}[Gradient damping]
The total coefficient \( w_0 + w_1 \leq 1 \) serves as an implicit gradient damping factor, reducing gradient variance. This is analogous to empirical findings in SNN training where the surrogate gradients are scaled by $\approx 0.3$ to improve gradient stability. 
\end{remark}

\begin{remark}[Bias-variance trade-off]
In the zero-temperature limit (\( k \to 0 \)), we have \( w_1 = w_0 = 0.5 \), yielding equal weight to both output states. This corresponds to the trapezoidal rule and is unbiased for linear or quadratic losses \cite{shekhovtsov2021biasvariancetradeoff}. 
\end{remark}

\vspace{0.5em}
\textbf{Generalising the IW-ST(\(p\)) family.}
The analytical form of AGR naturally motivates a broader family of estimators where the weights $w_0$ and $w_1$ can be chosen more flexibly.

\begin{definition}[Importance-weighted straight-through estimator (IW-ST(\(p\)))]
Let \( p \in (0,1) \) be a learned or fixed mixing parameter. IW-ST($p$) estimators approximate the finite difference as:
\begin{align}
\mathcal{L}_1-\mathcal{L}_1
\approx  \mathbb{E}_{o \sim \text{Bern}(p)} 
\left[ \frac{d\mathcal{L}}{d o} \right]
= p \cdot \left. \frac{d\mathcal{L}}{d o} \right|_{o = 1} 
+ (1 - p) \cdot \left. \frac{d\mathcal{L}}{d o} \right|_{o = 0}
\end{align}
\end{definition}

Let \( o \sim \text{Bern}(F(h)) \) be the forward-sampled output. The IW-ST(\(p\)) gradient can be estimated via importance sampling as:
\begin{align}
\mathbb{E}_{o \sim \text{Bern}(F(h))} \left[
\left( \frac{p}{F(h)} \right)^o 
\left( \frac{1 - p}{1 - F(h)} \right)^{1 - o} 
\cdot \frac{d\mathcal{L}}{d o}
\right]
\end{align}

\begin{remark}
Setting \( p = F(h) \) recovers the classical ST estimator, showing that IW-ST(\(p\)) generalises the ST estimator with tunable bias-variance behaviour. Choosing \( p \) dynamically based on \( h \) enables fine-grained control over the estimator’s properties per neuron and per forward pass.
\end{remark}

\textbf{Theoretical guarantees for IW-ST($p$) estimators.}
Whilst it has previously been established that the ST estimator is unbiased if the loss is multi-linear \cite{shekhovtsov2021}, when utilised in multi-layer networks, this assumption inevitably break down. \textbf{Here we present a new theoretical result that extend beyond these conditions by characterising the bias of IW-ST($p$) estimators in arbitrarily deep architectures and identifying a concrete criterion for unbiasedness based on the network state}. Note that in following we assume $F(h)$ is the normal CDF with variance $\sigma^2$, $F(h)=\Phi(\frac{h}{\sigma})$, where this variance can differ between neurons. When we condition on 'input' in the following, this refers to the input to the network (ie a given training example for supervised learning). Full details of these proofs are provided in Appendix C.

\begin{theorem}
    The ST method is unbiased when:
\begin{enumerate}
  \item For all neurons $i$ in any layer $l$, for all configurations of $o^{(l-1)}$,
  \[
    p\bigl(o^{(l)}_i=1 \mid o^{(l-1)},\text{input}\bigr)
    \;=\;
    \sum_k g\bigl(o_k^{(l-1)},\text{input}\bigr),
  \]
  i.e.\ each neuron’s probability of firing is a linear combination of its inputs.

  \item For all output neurons $j$ in layer $L$,
  \[
    p\bigl(o_j^{(L)} \mid o^{(L-1)},\text{input}\bigr)
    \;=\;
    \mathbb{E}_{\,o^{(L-1)}}\!\bigl[p\bigl(o_j^{(L)} \mid o^{(L-1)},\text{input}\bigr)\bigr].
  \]
  \item The loss function is linear in the output neurons:
  \[
   \mathcal{L}(o_i^{(L)}=1)- \mathcal{L}(o_i^{(L)}=0)=\frac{d\mathcal{L}}{do_i^{(L)}}.
  \]
\end{enumerate}
\end{theorem}
\begin{proof}[Proof sketch]
When computing the gradient with respect to neurons in layer~$l$, the key approximation to arrive at the ST estimator is
\[
  \mathbb{E}\bigl[\mathcal{L}\mid o^{(l+1)}\bigr]
  \;\approx\;
  \sum_j \mathbb{E}\bigl[\mathcal{L}\mid o_j^{(l+1)},\,o^{(l)}\bigr],
\]
for the observed $o^{(l)}$.  One shows that this holds at layer~$l+1$  if \(
  \mathbb{E}\bigl[\mathcal{L}\mid o^{(l+2)}\bigr]
  \;\approx\;
  \sum_j \mathbb{E}\bigl[\mathcal{L}\mid o_j^{(l+2)},\,o^{(l)}\bigr],
\) and $p(o_j^{(l+2)}=1|o^{(l+1)})$ is linear in $o^{(l+1)}$.  A recursive argument then establishes that conditions (1) and (2) are sufficient for unbiasedness when the finite difference is used at the final layer. Condition (3) enables us to replace this remaining finite difference with its linear approximation. 
\end{proof}

\begin{corollary}
Minimising the bias corresponds to maximising
\[
  p\!\Bigl(\bigcap_i\bigl\lvert\tfrac{h_i}{\sigma_i}-\mathbb{E}[\tfrac{h_i}{\sigma_i}]\bigr\rvert<\epsilon \;\Bigm|\;\text{input}\Bigr),
\]
for sufficiently small $\epsilon>0$, where $i$ ranges over all neurons.
\textbf{This corollary formalises a new sufficient condition for low-bias training in deep binary networks.}
\end{corollary}

\begin{proof}
Note that we have 
\[
  p\bigl(o_j^{(l+1)}=1\mid o^{(l)}\bigr)
  = F\!\bigl(h_j^{(l+1)}\bigr)
  = \Phi\!\bigl(\tfrac{h_j^{(l+1)}}{\sigma_j}\bigr)
\]
where $h_j^{(l+1)}=\sum_i w^{(l+1)}_{ji} o_i^{(l)}$. Consider a linear approximation of this Gaussian CDF in the terms $o^{(l)}$, around $\mathbb{E}[\tfrac{h_j^{(l+1)}}{\sigma_j}]$. Since the Gaussian CDF is continuous, its linear approximation will only be accurate in a small interval around $\mathbb{E}[\tfrac{h_j^{(l+1)}}{\sigma_j}]$:
\[
  \bigl\lvert\tfrac{h_j^{(l+1)}}{\sigma_j}
  - \mathbb{E}[\tfrac{h_j^{(l+1)}}{\sigma_j}]\bigr\rvert
  < \epsilon.
\]
Imposing this bound at every neuron makes the ST bias vanish as $\epsilon\to0$, as it guarantees condition (1) and (2) above are satisfied. 
\end{proof}

\emph{(In what follows we suppress conditioning on the input.)}

\begin{remark}
Since the curvature of the normal CDF is minimised at its inflection point (zero), if $\mathbb{E}[h_i/\sigma_i]=0$ for each $i$, one may select a larger $\epsilon$ with only a small increase in bias.
\end{remark}

\begin{proposition}
We have the following lower bound:
\[
  p\!\Bigl(\bigcap_i\bigl\lvert\tfrac{h_i}{\sigma_i}
  - \mathbb{E}[\tfrac{h_i}{\sigma_i}]\bigr\rvert<\epsilon\Bigr)
  \;\ge\;
  1 \;-\; \frac{\sum_i \mathrm{Var}(h_i/\sigma_i)}{\epsilon^2}.
\]
\end{proposition}

\begin{proof}
Application of the union bound and Chebyshev’s inequality.
\end{proof}

Thus, minimising $\sum_i \mathrm{Var}(h_i/\sigma_i)$ reduces ST bias.  Additionally enforcing a zero-mean constraint $\sum_i\bigl(\mathbb{E}[h_i/\sigma_i]\bigr)^2$  yields a MSE constraint, \(\sum_i \mathbb{E}\bigl[(h_i/\sigma_i)^2\bigr]\approx 0\)

\begin{remark}[Connection to vanishing gradients]
Using the ST backpropagation update with Gaussian PDF $\phi$,
\[
  \mathbb{E}[\mathcal{L}(o^{(l)}_i=1)-\mathcal{L}(o^{(l)}_i=0)]
  \approx
  \sum_k \frac{w^{(l+1)}_{ki}}{\sigma_k}\,
    \phi\!\bigl(\tfrac{h^{(l+1)}_k}{\sigma_k}\bigr)\,
    \mathbb{E}[\mathcal{L}(o_k^{(l+1)}=1)-\mathcal{L}(o_k^{(l+1)}=0)\mid o^{(l)}].
\]
Since $\phi(h_k/\sigma_k)$ decays for inputs far from zero, this can lead to the vanishing gradient problem \cite{hochreiter1997long}. The constraint $\mathbb{E}[(h_i/\sigma_i)^2]\approx0$ therefore both minimises ST bias and can help in preventing vanishing gradients.
\end{remark}

\begin{proposition}
When the zero–mean constraint is enforced for each neuron, the bias of IW-ST($p$) is minimised when
\[
  p \in
  \begin{cases}
    [\,0.5,\,1\,] &\text{if }F(h)>0.5,\\
    [\,0,\,0.5\,] &\text{if }F(h)<0.5.
  \end{cases}
\]
\end{proposition}
\begin{proof}
See Appendix C.
\end{proof}

\begin{proposition}
Under mild assumptions on the loss, and assuming that $F(h)$ is optimised to minimise this loss, the variance of IW-ST($p$) is minimised when
\[
  p \in
  \begin{cases}
    [\,F(h),\,1\,] &\text{if }F(h)>0.5,\\
    [\,0,\,F(h)\,] &\text{if }F(h)<0.5.
  \end{cases}
\]
\end{proposition}
\begin{proof}
See Appendix C.
\end{proof}

\begin{remark}
By setting $p=F(h)$, the classical ST estimator sits in both the low-bias and low-variance intervals.
\textbf{This provides a new justification for the ST estimator as a solution to the bias–variance trade-off in stochastic binary networks.}
\end{remark}

\section{Bayesian Spiking Neural Networks}

In Section 1, we introduced the IW-ST estimator family for noisy binary neural networks and derived conditions for reducing gradient bias and variance. We now show how Bayesian Neural Networks provide a noise source that enables end-to-end gradient-based optimisation in both BNNs and SNNs. This setting offers mechanisms to minimise gradient bias and variance, mitigate vanishing gradients, and ultimately support normalisation-free networks.

\textbf{Bayesian neural networks.}
Unlike conventional ANNs that learn fixed parameters, Bayesian Neural Networks assign prior distributions to parameters and aim to compute their posterior given the data. Since exact inference is generally intractable, we use variational inference to find an approximate posterior, 
\(q(W|data)\), that maximises the evidence lower bound (ELBO), also referred to as the (negative) variational free energy (VFE):
%\vspace{-0.05in}
\begin{align}
\text{ELBO} = \sum_d \mathbb{E}_{W \sim q(W)} \left[ \log p\left(\text{target}^{(d)} \mid \text{input}^{(d)}, W \right) \right] + D_{\mathrm{KL}}\left(q(W) \,\|\, p(W)\right)
\end{align}
%\vspace{-0.03in}
where \((input^{(d)},target^{(d)})\) is the \(d^{th}\) data point, and \(D_{KL}(\cdot)\) denotes the Kullback-Leibler (KL) divergence.  To approximate the posterior, we parameterise the variational posterior and optimise its parameters via the ELBO. Since the likelihood term aligns with the loss function,gradient descent on the ELBO corresponds to optimising a supervised loss plus a regularisation term (e.g., weight decay).

\textbf{Local reparameterisation.}
Whilst we can calculate the gradient of the ELBO in its current form, to use the IW-ST($p$) estimator we require that the expectation is over the output of each neuron, not the weights. Fortunately, we can easily make this change via the local reparameterisation trick \cite{kingma2015variational}. If we assume \(w_{ij} \sim_{iid} N(m_{ij},\sigma_{ij}^2)\), we can alternatively interpret the output as being noisy \(o_i \sim Bernoulli(F(h_i))\), where $F(h_i)=\Phi(\frac{h_i}{\sigma_i})$ with $\Phi()$ being the standard normal CDF, and we define $h_i=\sum_j m_{ij}^{(l)} o_j^{(l-1)}$, $\sigma_i^2:= \sum_j o_j^{(l-1)2} \sigma_{ij}^2$.

%This local reparameterisation reveals the crucial role of the learnt variance parameters associated with each weight. These parameters not only control the level of noise in the forward pass, but they also dictate the shape of $F()$, which as seen in our earlier derivations, has significant implications on learning in the backward pass. If the variance parameters are too small, $F'()$ will be too narrow, and no learning may occur as the gradient vanishes. On the other hand, if the variance parameters are too large, the forwards pass will be overwhelmed by noise. We now turn our attention to the KL divergence term, which plays a crucial role in balancing these two considerations.

\textbf{Prior and posterior parametrisation.}
We use a variational posterior in which each of the weights are independent normal random variables, \(w_{ij} \sim_q  N(m_{ij},\sigma_{ij}^2)\), where $\sim_q$ is used to denote that this is the approximation posterior with density $q$. Additionally, if we assume the prior distribution is also the product of independent normal random variables, \(w_{ij} \sim_p N(\alpha_{ij},\tau_{ij}^2)\), the KL divergence is tractable, and is given by:
%\vspace{-0.05in}
\begin{equation}
D_{\mathrm{KL}}(q(W) \,\|\, p(W)) 
= \sum_{i,j} \ln\left( \frac{\tau_{ij}}{\sigma_{ij}} \right) 
+ \frac{(m_{ij} - \alpha_{ij})^2 + \sigma_{ij}^2 - \tau_{ij}^2}{2 \tau_{ij}^2}
\end{equation}
%\vspace{-0.01in}
For the prior variance, \(\tau_{ij}^2\), we follow \cite{kharitonov2018variationaldropoutempiricalbayes}, and apply empirical Bayes to optimise \(\tau_{ij}^2\). Using the ELBO,  we can find a closed form solution for the optimal \(\tau_{ij}^2\) by simply minimising the expression for \(D_{KL}({q(w_{ij})}|p(w_{ij}))\) given above.  This gives the  solution \(\tau_{ij}^2=\frac{m_{ij}^2+\sigma_{ij}^2}{2}\), which when substituted back into the KL divergence gives a simplified regularisation term:
%\vspace{-0.05in}
\begin{align}
D_{\mathrm{KL}}(q(w_{ij}) \,\|\, p(w_{ij})) 
= \ln\left( \frac{\sqrt{m_{ij}^2 + \sigma_{ij}^2}}{2\sigma_{ij}} \right) 
= \frac{1}{2} \ln\left( 1 + \left( \frac{m_{ij}}{\sigma_{ij}} \right)^2 \right) + c
\end{align}
The term acts to minimise \((\frac{m_{ij}}{\sigma_{ij{}}})^2\) - regularising each weight in relation to its standard deviation.

\textbf{The role of the KL divergence.}
The variance parameter must balance between too much and too little noise. As performance improves, the likelihood term pushes the noise variance toward $0$. However, if the noise and hence variance approaches $0$, gradients vanish and learning stalls. The KL divergence term counteracts this by applying upward pressure on $\sigma_{ij}^2$, preventing it from collapsing, as it penalises \((\frac{m_{ij}}{\sigma_{ij}})^2\).

Recall from Section 1 that minimising ST bias requires  $E[(\frac{ h_i}{\sigma_i})^2] \approx 0$. While the current KL divergence term may loosely align with this, it doesn’t enforce it directly. It regularises the signal/noise ratio per weight, whereas bias is reduced by controlling this ratio per neuron. As a result, it may over-regularise without ensuring low bias. To address this, we can instead place the variational posterior on the sum of active weights rather than on all weights, that is: 
\begin{equation}
\sum_j w_{ij}^{(l)} o_j^{(l-1)} \,\big|\, \text{input},\, o^{(1:l-1)} 
\sim_q \mathcal{N}\left(
\sum_j m_{ij} o_j^{(l-1)},\ 
\sum_j \sigma^2_{ij} \, o_j^{(l-1)2}
\right)
\end{equation}
%\vspace{-0.12in}
and then the ELBO objective becomes:
\begin{equation}
\text{ELBO} = \sum_d \mathbb{E}_{W \sim q} \left[
\log p\left(\text{target}^{(d)} \mid \text{input}^{(d)}, W\right)
\right]
+ \frac{1}{2} \sum_i \ln\left(1 + \left( \frac{\sum_j m_{ij} o_j}{\sigma_i} \right)^2 \right) + c
\end{equation}
%\vspace{-0.05in}
where the second term now constrains $\sum_i g(E[(\frac{h_i}{\sigma_{i}})^2])$, where $g()$ is strictly increasing.

A downside of this approach is that we lose the previous interpretation of individual weight posteriors. Instead, the probabilistic interpretation now applies to the posterior over the sum of active weights. However, since each active weight still contributes its mean $m_{ij}$ and variance $\sigma_{ij}^2$, we can still disentangle each weight’s contribution to the overall posterior.

\textbf{Local reparameterisation for spiking neurons.}
The approach described can be applied to various architectures with binary activations. To use the IW-ST estimator, we must apply the local reparameterisation trick to the noisy weights so that each neuron's output can be treated as a Bernoulli random variable, with a probability that is a differentiable function of the network parameters. Two conditions ensure this: (i) the preactivation is a linear function of the weights, and (ii) the activation function is applied directly to this linear combination.
This holds for fully connected layers, and we address weight-sharing in convolutional layers in the Appendix D. For SNNs, since both feedforward weights $w_{ij}^{(l)}$ and recurrent weights $r_{ik}^{(l)}$ enter the membrane potential linearly, and the output is a Heaviside function of that potential, both conditions are satisfied. However, the recurrent structure induces dependencies between $o^{(l)}_{i,t}$ and $o^{(l)}_{i,s}$ for all times $t,s$, making the Bernoulli probability—needed for differentiation—difficult to work with. To address this, we adopt an alternative "convolutional view" of SNNs by expanding the recurrence relation to express the membrane potential as:
%\vspace{-0.05in}
\begin{equation}
h^{(l)}_{i,t} = \sum_s \beta^s \left(
\sum_j w^{(l,t,s)}_{ij} \, o^{(l-1)}_{j,t-s} 
+ \sum_{k \ne i} r^{(l,t,s)}_{ik} \, o^{(l)}_{k,t-s-1}
- \theta \, o^{(l)}_{i,t-s-1}
\right)
\end{equation}
%\vspace{-0.01in}
This approach involves both sampling independent weights at each time step and resampling past weights, $w_{ij}^{(l,t,s)} \sim N(m_{ij}^{(l)},\sigma_{ij}^{2(l-1)})$, ensuring outputs at different time steps are independent and making the probability function tractable per neuron and time step. Crucially, this can still be implemented efficiently using the recurrent view, without explicitly resampling all past weights. A full derivation is provided in Appendix D, but the key idea is to express membrane potential as:
%\vspace{-0.05in}
\begin{align}
h^{(l)}_{i,t} &= h^{*(l)}_{i,t} + \kappa^{(l)}_{i,t} \, \epsilon^{(l)}_{i,t} \nonumber \\
h^{*(l)}_{i,t} &= \beta h^{*(l)}_{i,t-1} 
+ \sum_j m^{(l)}_{ij} \, o^{(l-1)}_{j,t} 
+ \sum_{k \ne i} p^{(l)}_{ik} \, o^{(l)}_{k,t-1} 
- \theta \, o^{(l)}_{i,t-1} \nonumber \\
\kappa^{(l)2}_{i,t} &= \beta^2 \kappa^{(l)2}_{i,t-1} 
+ \sum_j \sigma_{ij}^{(l)2} \, o_{j,t}^{(l-1)2} 
+ \sum_{k \ne i} \nu_{ik}^{(l)2} \, o_{k,t-1}^{(l)2}
\end{align}
% \vspace{-0.05in}
where $h^{*(l)}_{i,t}$ is the "noiseless" membrane potential, to which we add a normal random variable with variance $\kappa_{i,t}^{(l)2}$ to arrive at the membrane potential $h_{i,t}^{(l)}$. Both the noiseless membrane potential $h^{*(l)}_{i,t}$ and noise variance $\kappa^{(l)2}_{i,t}$ follow recurrence relations, depending only on the previous state and current inputs. In practice, using a single variance per layer—separately for feedforward and recurrent weights—performs well  ($\sigma_{ij}^{(l)}=\sigma^{(l)}$). With shared variance per layer, the forward-pass overhead is minimal— equal to adding one extra neuron per layer—making it negligible in most modern architectures. Without weight sharing, however, computing separate variances per weight doubles.
 
\textbf{Connection to the SG method.}
This variational inference scheme reduces to the SG method under four conditions: (i) fixed posterior variance, (ii) mean-field approximation (MFA) using weight means instead of sampling, (iii) omission of the KL divergence term, and (iv) SG width increasing with the number of presynaptic inputs.
% \vspace{-0.15in}
\section{Experiments}
% \vspace{-0.12in}
We now validate this Bayesian approach on BNNs and SNNs, comparing its training efficiency and generalisation abilities against state of the art surrogate-gradient based methods. We use the classical ST estimator in these experiments, but investigate the broader class of IW-ST($p$) estimators and also report full experimental details in Appendix E.

\textbf{Binary neural networks.}
We evaluate our Binary Bayesian Neural Network (BBNN) framework on the CIFAR-10 dataset \cite{krizhevsky2009learning}, using a SG model with batch normalisation as the baseline. We compare this with four BBNN variants: (i) full BBNN with learned posteriors, (ii) BBNN-MFA using mean-field approximation, (iii) BBNN-FPV with fixed posterior variance and MFA, and (iv) BBNN-NKL, which omits the KL term alongside using MFA and fixed variance.

%For the network architecture, we use a binary variant of ResNet-20, trained for 200 epochs with Adam (ref). Full hyperparameter choices are provided in the supplementary materials.

We train a binary ResNet-20 architecture for 200 epochs using Adam (hyperparameters in the supplementary materials). Figure~\ref{fig:all}.a shows training loss and test accuracy from a representative run, with consistent results across seeds (see Appendix E for tables with error bars).
All variants except BBNN-NKL reach nearly 90\% test accuracy; BBNN-NKL fails to train, remaining below 50\%, highlighting the KL term’s role in stabilising training without normalisation.  The remaining BBNNs perform similarly, with BBNN-FPV matching the baseline SG model while avoiding normalisation. Full BBNN yields the lowest test loss, suggesting better-calibrated predictions from its Bayesian formulation. Overall, BBNN-FPV strikes a strong balance between performance and simplicity, requiring neither sampling nor learned variance. We provide further analysis on the impact of shared variance parameters in Appendix E, as well as further analysis on the role of the KL divergence in mitigating vanishing gradients. 

\begin{figure}
    \centering
    \includegraphics[width=0.9\linewidth]{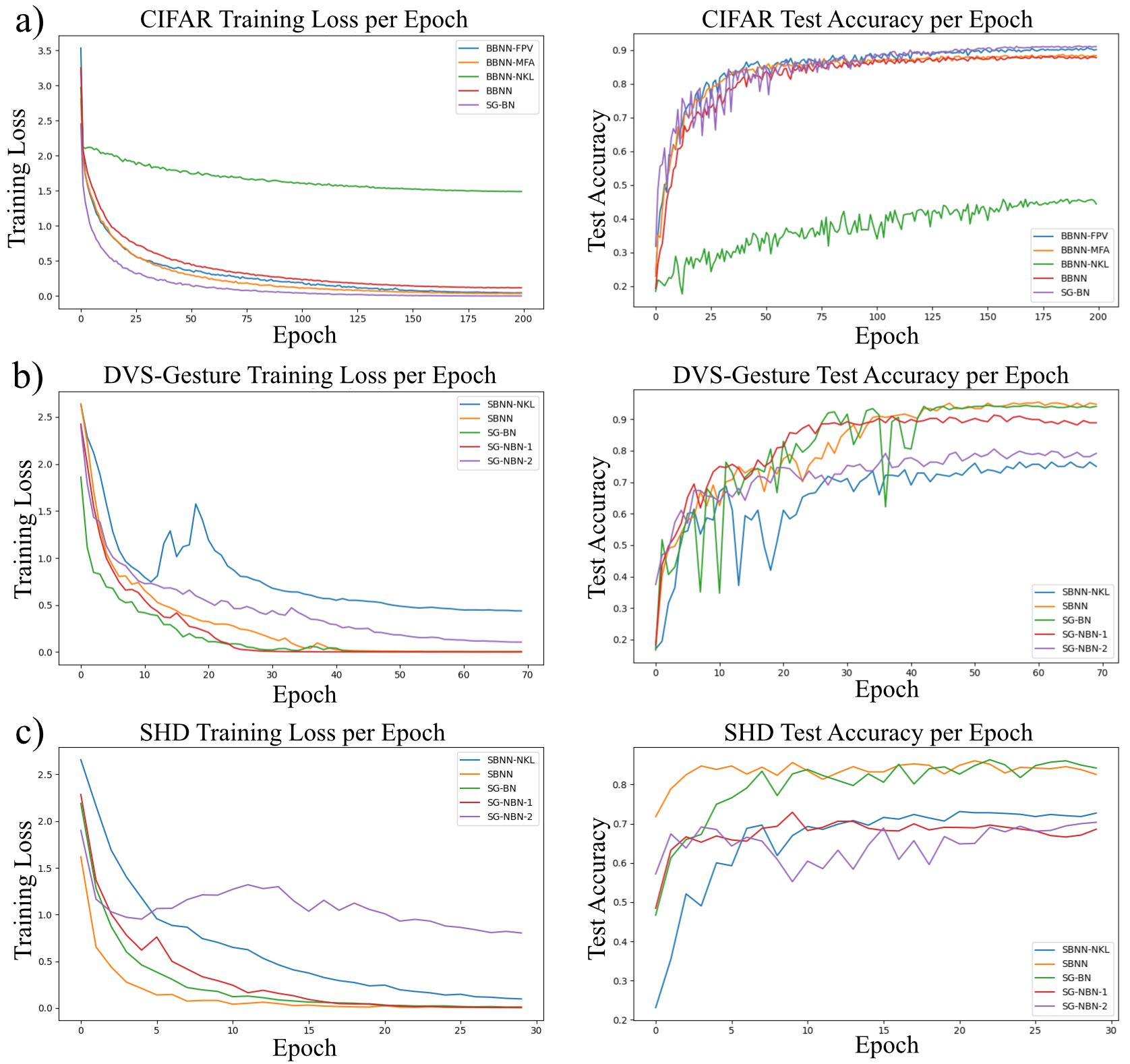}
    \caption{Training loss and test accuracy across epochs for \textbf{a)} CIFAR-10, \textbf{b)} DVS-Gesture, and \textbf{c)} SHD datasets}
    \vspace{-0.5cm}
    \label{fig:all}
\end{figure}

\textbf{Spiking neural networks.}
We evaluate Spiking Bayesian Neural Networks (SBNNs) on two event-based benchmarks: DVS-Gesture (DVS, \cite{amir2017low}) and Spiking Heidelberg Digits (SHD, \cite{cramer2020heidelberg}), both requiring temporal inference and well-suited for SNNs. We compare three SG baselines—one with normalisation (SG-BN) and two without with different SG widths (SG-NBN-1, SG-NBN-2)—against our Bayesian SNN (SBNN) and its KL-free variant (SBNN-NKL).%For the normalisation scheme, we use batch normalisation through time (BNTT), which applies a separate batch-norm layer to the weighted inputs to each neuron at each time step. In this setting we report the results for BBNN and BBNN-NKL which are defined above. 

On DVS, we use a spiking ResNet-20 unrolled over 50 time steps. As shown in Figure~\ref{fig:all}.b, SBNN-NKL again performs poorly, highlighting the KL term’s role in stable training. SG-NBN can approach SG-BN performance when its surrogate width is carefully tuned, suggesting normalisation isn’t strictly necessary with well-chosen parameters. However, SBNN achieves comparable performance to SG-BN without such tuning, working effectively out of the box.
%For the DVS dataset, we use a spiking variant of ResNet-20, with 50 time steps. The evolution of the training loss and test accuracy across epochs for one sample training run is visualised in figure (x). We can observe that the KL divergence term is again essential to ensure efficient learning for the Bayesian approach, with the SBNN-NKL approach significantly worse than the full SBNN. We also observe that when the SG shape is optimised as a hyperparameter, the SG-NBN approach is competitive with SG-BN approach in terms of training efficiency, highlighting that BNTT is not always required for efficient learning in SNNs, even in deeper networks. Crucially however, our SBNN method does not require this choice to be made, and achieves higher accuracy on par with SG-BN.
Unlike in BNNs, training in SNNs without normalisation or KL does not completely stall, thanks to the temporal dimension: each example provides multiple learning signals over time. A neuron only needs to spike near threshold once to contribute a gradient, whereas BNN neurons must do so in a single pass, making them more prone to vanishing gradients. Thus, while KL regularisation is essential in BNNs, its role in SNNs is subtler but still helpful.

On SHD, we use fully recurrent SNNs with two hidden layers of 256 neurons over 100 time steps. While recurrence boosts performance, it often destabilises training. Figure~\ref{fig:all}.c shows that SBNN converges efficiently and matches the generalisation of SG-BN. In contrast, SBNN-NKL fails to learn, and SG-NBN struggles to generalise despite reduced training loss. These results suggest that both batch normalisation and KL regularisation aid generalisation, which SBNN achieves implicitly through its Bayesian formulation.

\section{Limitations}

This work advances our understanding of surrogate gradient (SG) methods by grounding them in a principled gradient estimate derived from a corresponding noisy network. However, several limitations remain, motivating future directions.

First, while we show that SG methods approximate the expected gradient in a noisy network via a mean-field approximation, its empirical performance often surpasses than that of Monte Carlo sampling, even though the latter in theory offers an unbiased estimate. We hypothesise that this apparent contradiction arises because the mean-field approximation induces a deterministic forward pass, and the bias induced by this approximation is offset by the substantial reduction in variance. A formal theoretical treatment is needed to quantify this trade-off and establish when mean-field estimates can outperform unbiased sampling. 

Second, our analysis invites a deeper question: what are the limitations of using the gradient from an assumed noisy network? There are method to calculate the exact gradient in SNNs by exploiting the continuous spike times \cite{wunderlich2021event}, yet these methods underperform compared to the surrogate gradient method. Whilst assuming the network is noisy may provide a number of benefits such as more robust learning, for applications where precise spike-timing is important, the presence of noise may degrade performance and limit the expressivity of the spiking network. Exploring alternative gradient estimation strategies which do not rely on noise may enable SNNs to fully exploit their temporal processing capabilities. 

From a normalisation perspective, although we demonstrate that KL divergence mitigates vanishing gradients in deep-residual and shallow-recurrent networks, it fails to stabilise training in deep feedforward architectures without residual connections. Whilst residual connections are standard in modern deep learning, the reliance on these nonetheless suggest traditional normalisation schemes remains indispensable when residual connections are absent. Moreover, our reliance on tuning the KL divergence weight by treating it as a hyperparameter, while standard in variational inference,  detracts from a fully-Bayesian treatment and weakens the probabilistic interpretation of the model.

Together, these limitations point to several promising avenues: developing theoretically grounded variance-bias trade-offs for SG; constructing noise-free but effective gradient estimators; and establishing self-tuning, fully Bayesian mechanisms for regularisation and normalisation.

\section{Acknowledgments}

J.W., M.K. and A.R. are funded by the Office of National Intelligence grant (Ref: NI230100153). A.R. and M.K. are funded by the Australian National Health and Medical Research Council Investigator Grant (Ref: 1194910). A.R. is a CIFAR Azrieli Global Scholar in the Brain, Mind \& Consciousness Program. A.R. is also affiliated with The Wellcome Centre for Human Neuroimaging, supported by core funding from Wellcome [203147/Z/16/Z].

\bibliographystyle{unsrtnat}
\bibliography{references}

\clearpage 

\section*{Appendix}
In Appendix A, we clarify the connection between Surrogate Gradient (SG) method and the Straight-Through (ST) estimator, correcting common misconceptions. Appendix B reviews continuous-relaxation based gradient approaches, culminating in the Analytical Gumbel Rao (AGR) estimator and related IW-ST estimator families. Appendix C provides the key results from the paper on the bias and variance of IW-ST estimators. Appendix D provides how Bayesian Neural Networks can be applied in the context of convolution and spiking layers, proving the key recurrence relation stated in the main text. Appendix E provides a number of additional experiments, full results and training details. 

\section*{A - Surrogate Gradient Method, Straight-Through Estimator}

\subsubsection*{Surrogate Gradient Method}
Deterministic binary neural networks (BNNs) use non-differentiable Heaviside activation \( H'(x) = \mathbb{I}(x \geq 0) \) which causes the gradient to vanish almost everywhere:

\begin{align}
\frac{d\mathcal{L}}{d w_{ij}^{(l)}} 
= \frac{d\mathcal{L}}{d o_i^{(l)}} \cdot \frac{d o_i^{(l)}}{d h_i^{(l)}} \cdot \frac{d h_i^{(l)}}{d w_{ij}^{(l)}}
= \frac{d\mathcal{L}}{d o_i^{(l)}} \cdot H'(h_i^{(l)}) \cdot o_j^{(l-1)}
\end{align}

To address this, the SG method replaces the non-differentiable \( H'(x)\) with derivative of a smooth surrogate function \( S(x) \), such as a sigmoid or triangular function \cite{Neftci2019}.
%Let \( S(x) \) be a smooth, monotonically increasing function such that \( S(x) \approx H(x) \) and \( S'(x) \approx \delta(x) \). 
The surrogate satisfies:

\begin{equation*}
S(x) \approx H(x) \quad \text{and} \quad S'(x) \approx \delta(x)
\end{equation*}

In practice, this results in the following approximation in the backward pass:

\[
\frac{d o_i^{(l)}}{d h_i^{(l)}} = H'(h_i^{(l)}) \Rightarrow S'(h_i^{(l)})
\]
while retaining the binary forward pass:
\[
o_i^{(l)} = H(h_i^{(l)}), 
\]
but
\[
\frac{d \mathcal{L}}{d h_i^{(l)}} = \frac{d \mathcal{L}}{d o_i^{(l)}} \cdot S'(h_i^{(l)})
\]
The upstream gradient from the next layer is computed recursively:
\begin{align}
\frac{d\mathcal{L}}{d o_i^{(l)}} 
= \sum_k \frac{d\mathcal{L}}{d o_k^{(l+1)}} \cdot \frac{d o_k^{(l+1)}}{d h_k^{(l+1)}} \cdot \frac{d h_k^{(l+1)}}{d o_i^{(l)}} 
= \sum_k \frac{d\mathcal{L}}{d o_k^{(l+1)}} \cdot S'(h_k^{(l+1)}) \cdot w_{ki}^{(l+1)}
\label{eqSG} \end{align} 

\subsubsection*{Straight-Through Estimator}
The term \textit{Straight-Through (ST)} estimator has been used ambiguously in the literature. In some cases, it refers to a crude estimator where the non-differentiable function is replaced with a constant derivative, typically $H'(x) \approx 1$ \cite{bengio2013estimating}. In other cases, as used in this work, it refers to a more principled gradient estimator derived for Bernoulli neurons using a linear approximation to the finite difference \cite{shekhovtsov2021}. 

Note that the definition that replaces the non-differentiable function with an identity can be viewed as a special case of surrogate gradient method with is $S(x)=x$. Throughout this paper, ST refers to the gradient estimator for Bernoulli neurons based on a probabilistic model rather than heuristic replacement.
%In what follows, and in the main text, ST always refers to the gradient estimator in networks of Bernoulli neurons.
 
As discussed in the main text, the ST estimator is equivalent to the SG method in Eq.~\ref{eqSG} when the surrogate function $S(x)$ is chosen as the CDF of a continuous random variable $F(x)$. This insight unifies the deterministic SG approaches and probabilistic interpretations: even when used heuristically, SG can be viewed as approximating the true gradient in a noisy network governed by the distribution $F(x)$.

\section*{B - Continuous relaxations, Analytic Gumbel-Rao, IW-ST}
Here we briefly review the continuous relaxation approach before deriving the Analytic Gumbel-Rao estimator in full detail. 

\subsubsection*{Continuous relaxation gradients}
The gradient we wish to calculate is given by 
\begin{equation}
  \frac{d}{d w_{ij}}
  \;\mathbb{E}_{\,o^{(1:L)}\sim p_W}\bigl[\mathcal{L}\bigr]
\end{equation}

where we have written \(o^{(1:L)} \sim p_W\) to make explicit the dependence of the joint distribution of neuron outputs on the network weights. To perform backpropagation, we first need to take the derivative inside the expectation, however this exchange is not yet valid as the distribution depends on the weights, which is what we are taking the derivative with respect to. This is easily handled by the reparameterisation trick \cite{kingma2014auto}, which reparameterises the random variables so that the distribution no longer depends on the weights. However if the distribution is not continuous, this reparameterisation alone is not sufficient, as the reparameterised variable will involve some discontinuous function which still prevents backpropagation. The solution to this issue is to replace the discontinuous function with some continuous approximation. This process of reparameterisation and subsequent continuous relaxation is referred to as the Gumbel-Softmax trick, introduced in \cite{maddison2017concrete, jang2017categorical}. Applying the Gumbel-Softmax in the Bernoulli case, we have:

\begin{align}
  o_i^{(l)}
  &= H\!\Bigl(\ln\!\bigl(\tfrac{F(h_i^{(l)})}{1 - F(h_i^{(l)})}\bigr) + \eta\Bigr) \approx S_k\!\Bigl(\ln\!\bigl(\tfrac{F(h_i^{(l)})}{1 - F(h_i^{(l)})}\bigr) + \eta\Bigr)
\end{align}

where the first step is reparameterising the Bernoulli, and the second is the continuous relaxation. The Gumbel-Softmax approach uses a logistic random variable \(\eta \sim Logistic(0,1)\) and  \(S_k(x)=\frac{1}{1+e^{-\frac{x}{k}}}\) is a sigmoid, which converges to the Heaviside as the temperature $k$ goes to $0$.  Whilst this is the standard Gumbel-Softmax reparameterisation, other reparameetrisations are possible - we proceed with the following uniform reparameterisation, as it enables analytical Rao-Blackwellisation in the next section: 

\begin{align}
  o_i^{(l)}
  &= H\bigl(F(h_i^{(l)}) - 1 + u_i\bigr) \approx S_k\bigl(F(h_i^{(l)}) - 1 + u_i\bigr)
  \;:=\;
  s_i^{(l)}
\end{align}

where \(u_i \sim Uniform[0,1]\).

We can now take the gradient inside the expectation and perform backpropagation, ultimately yielding the following gradient estimate:

\begin{align}
  \nabla=
      \frac{d\mathcal{L}^s}{ds_i^{(l)}}
      \;S_k'\bigl(F(h_i^{(l)}) - 1 + u_i\bigr)
      F'\bigl(h_i^{(l)}\bigr)\
      s_j^{(l-1)}.
\end{align}

where $\nabla$ denotes the estimate of the gradient. We use \(\mathcal{L}^s\) here to denote the loss calculated based on the output of the network with the smooth approximations.  There bias-variance trade-off here is straightforward: a higher temperature corresponds to lower variance but high bias, whilst low temperature corresponds to low bias but high variance \cite{shekhovtsov2021biasvariancetradeoff}. The problematic aspect of this approach is that we replace \(\mathcal{L}\) with \(\mathcal{L}^s\), and  \(o_j^{(l-1)}\) with \(s_j^{(l-1)}\) ,  and the gradient will only be unbiased when $\mathcal{L}^s=\mathcal{L}$ and $s^{(l-1)}=o^{(l-1)}$, which is to say when the spiking behaviour of all other neurons in the network is preserved and smooth approximations are not used in the forward pass. In other words, applying a continuous relaxation to a single neurons output will reduce the gradient variance when we backpropagate through that neuron,  but at the cost of increasing the bias when we backpropagate through \textit{all} other neurons. Additionally, the need to maintain a continuous network for training is cumbersome, preventing online-learning and losing the energy efficiency of the binary counterpart. 

\subsubsection*{Gumbel-Softmax Straight-Through Estimator}
A straightforward solution to these issues is the Gumbel-Softmax Straight Through Estimator (GS-ST) \cite{jang2017categorical}. This approach uses the same equations for the backwards pass as the GS estimator, but uses the values from a binary forward pass rather than the values from a forward pass through the relaxed network: 

\begin{equation}
 \nabla =  \frac{d\mathcal{L}}{ds_i^{(l)}}\Bigm|_{s_i^{(l)}=o_i^{(l)}}
  \;S_k'\bigl(F(h_i^{(l)}) - 1 + u_i\bigr)\;
  F'\bigl(h_i^{(l)}\bigr)\;
  o_j^{(l-1)}. 
\end{equation}

This approach is no longer unbiased as $k$ goes to $0$, but avoids the steep bias-variance trade-off of the GS estimator. In what follows we write $\frac{d\mathcal{L}}{do^{(l)}_i}:=\frac{d\mathcal{L}}{ds^{(l)}_i}|_{s^{(l)}_i=o_i^{(l)}}$

\subsubsection*{Gumbel-Rao Estimator}
To further reduce the variance of this estimator,  Rao-Blackwellisation can be applied as in \cite{paulus2020}:

\begin{align}
 \nabla= &\mathbb{E}_{\,u_i \mid o^{(1:L)}}
  \Bigl[
    \frac{d\mathcal{L}}{d o_i^{(l)}}
    \,S'\bigl(F(h_i^{(l)}) - 1 + u_i\bigr)\,
    F'\bigl(h_i^{(l)}\bigr)\,
    o_j^{(l-1)}
  \Bigr]
  \nonumber\\
  &= 
  \frac{d\mathcal{L}}{d o_i^{(l)}}
  \;\mathbb{E}_{\,u_i \mid o^{(1:L)}}
    \bigl[S'\bigl(F(h_i^{(l)}) - 1 + u_i\bigr)\bigr]\,
  F'\bigl(h_i^{(l)}\bigr)\,
  o_j^{(l-1)}.
\end{align}

Note that the only term that depends on $u_i$ is \(S'(F(h_i^{(l)})-1+u_i)\), and the other terms are constant conditional on $o^{(1:L)}$. When the Gumbel-Softmax relaxation is used, this expectation is intractable, and is in practice estimated using Monte Carlo samples. 

\subsubsection*{Analytical Gumbel-Rao}
Conveniently, due to our choice of uniform reparameterisation, we can evaluate this expectation analytically. Since $o_i^{(l)}=H(F(h_i)-1+u_i)$, we can infer that $u_i|o^{(1:l-1)},o_i^{(l)}=1 \sim Uniform[1-F(h_i^{(l)}),1]$ and $u_i|o^{(1:l-1)},o_i^{(l)}=0 \sim Uniform[0,1-F(h_i^{(l)})]$.  These follow from the fact that when we observe the output of the neuron as being $0$ or $1$, this truncates the distribution $u$. We can therefore write the expectation as:

\begin{align}
  \mathbb{E}_{\,u_i \mid o^{(1:L)}}\bigl[S'\bigl(F(h_i^{(l)}) - 1 + u_i\bigr)\bigr]
  &= 
  \int_{0}^{1}
    S'\bigl(F(h_i^{(l)}) - 1 + u_i\bigr)
    \;p\bigl(u_i \mid o^{(1:L)}\bigr)\,
    \mathrm{d}u_i.
\end{align}

When $o_i^{(l)}=1$, this simplifies to:

\begin{align}
  \mathbb{E}_{\,u_i \mid o^{(1:L)}}\bigl[S'\bigl(F(h_i^{(l)})-1+u_i\bigr)\bigr]
  &= 
  \frac{1}{F(h_i^{(l)})}
  \int_{\,1 - F(h_i^{(l)})}^{\,1}
    S'\bigl(F(h_i^{(l)}) - 1 + u_i\bigr)
  \,\mathrm{d}u_i
  \\[1ex]
  &= 
  \frac{1}{F(h_i^{(l)})}
  \Bigl(
    S\bigl(F(h_i^{(l)})\bigr)
    \;-\;
    S(0)
  \Bigr)
\end{align}

On the other hand, when $o_i^{(l)}=0$ we have the simplification:

\begin{align}
  \mathbb{E}_{\,u_i \mid o^{(1:L)}}\bigl[S'\bigl(F(h_i^{(l)}) - 1 + u_i\bigr)\bigr]
  &= 
  \frac{1}{1 - F(h_i^{(l)})}
  \int_{0}^{\,1 - F(h_i^{(l)})}
    S'\bigl(F(h_i^{(l)}) - 1 + u_i\bigr)
  \,\mathrm{d}u_i
  \\[1ex]
  &= 
  \frac{1}{1 - F(h_i^{(l)})}
  \Bigl(
    S(0)
    \;-\;
    S\bigl(1 - F(h_i^{(l)})\bigr)
  \Bigr).
\end{align}

We can then write the result compactly as 

\begin{equation}
  \mathbb{E}_{u_i\mid o^{(1:L)}}\bigl[S'\bigl(F(h_i^{(l)}) - 1 + u_i\bigr)\bigr]
  =
  \Bigl(\frac{S\bigl(F(h_i^{(l)})\bigr)\;-\;S(0)}{F(h_i^{(l)})}\Bigr)^{o_i^{(l)}}
  \;\;
  \Bigl(\frac{S(0)\;-\;S\bigl(F(h_i^{(l)})-1\bigr)}{1 - F(h_i^{(l)})}\Bigr)^{1 - o_i^{(l)}}.
\end{equation}

This then gives the full gradient estimator: 

\begin{equation}
\nabla=  \left(
    \frac{S\bigl(F(h_i^{(l)})\bigr) - S(0)}
         {F(h_i^{(l)})}
  \right)^{o_i^{(l)}}
  \,
  \left(
    \frac{S(0) - S\bigl(F(h_i^{(l)}) - 1\bigr)}
         {1 - F(h_i^{(l)})}
  \right)^{1 - o_i^{(l)}}
  \;
  \frac{d\mathcal{L}}{d o_i^{(l)}}\,
  F'\bigl(h_i^{(l)}\bigr)\,
  o_j^{(l-1)}\bigr.
\end{equation}

Note that this estimator is a single-sample Monte Carlo estimate of its expectation. Writing down the first three terms, we arrive at an elegant interpretation of this estimator (we have dropped sub and super-scripts for clarity):
\begin{align}
  &\mathbb{E}_{\,o \sim \mathrm{Bern}\bigl(F(h)\bigr)}\Biggl[
    \Bigl(\tfrac{S\bigl(F(h)\bigr)-S(0)}{F(h)}\Bigr)^{o}
    \,
    \Bigl(\tfrac{S(0)-S\bigl(F(h)-1\bigr)}{1-F(h)}\Bigr)^{1-o}
    \,
    \frac{d\mathcal{L}}{d o}
  \Biggr] 
  \nonumber\\
  &= 
  \bigl(S\bigl(F(h)\bigr)-S(0)\bigr)\,\frac{d\mathcal{L}}{d o}\Bigm|_{o=1}
  \;+\;
  \bigl(S(0)-S\bigl(F(h)-1\bigr)\bigr)\,\frac{d\mathcal{L}}{d o}\Bigm|_{o=0}
  \nonumber\\
  &= 
  \bigl(w_1 + w_0\bigr)\,
  \Bigl(
    \tfrac{w_1}{w_1 + w_0}\,\frac{d\mathcal{L}}{d o}\Bigm|_{o=1}
    \;+\;
    \tfrac{w_0}{w_1 + w_0}\,\frac{d\mathcal{L}}{d o}\Bigm|_{o=0}
  \Bigr)
  \nonumber\\
  &= 
  \bigl(w_0 + w_1\bigr)\,
  \mathbb{E}_{\,o \sim \mathrm{Bern}\bigl(\tfrac{w_1}{w_1 + w_0}\bigr)}
  \Bigl[\frac{d\mathcal{L}}{d o}\Bigr],
\end{align}

where $w_0=S(0)-S(F(h)-1)$ and  $w_1=S(F(h))-S(0)$

Note that $w_0+w_1\leq1$ and $w_1,w_2\geq0$. As such, we can interpret the expectation term as performing a trapezoidal approximation of the true finite difference:

\begin{align}
  \mathcal{L}_1 - \mathcal{L}_0
  &= 
  \int_{0}^{1} \frac{d\mathcal{L}}{d o}\,d o \approx 
  \frac{w_0}{w_0 + w_1}\,
    \frac{d\mathcal{L}}{d o}\Bigm|_{o=0}
  \;+\;
  \frac{w_1}{w_0 + w_1}\,
    \frac{d\mathcal{L}}{d o}\Bigm|_{o=1}
  \\[0.75ex]
  &=
  \mathbb{E}_{\,o \sim \mathrm{Bern}\bigl(\tfrac{w_1}{w_0 + w_1}\bigr)}
  \Bigl[\frac{d\mathcal{L}}{d o}\Bigr].
\end{align}

We have omitted the leading term $w_0+w_1 \leq 1$ to clearly illustrate the connection to the trapezoidal rule for estimating the integral. However, as noted in Remark 1, this term serves as a 'dampening factor' that reduces the variance of the gradient estimator. The estimator itself is  a linear approximation, where weights are assigned to each terminal value ($o=0$ and $o=1$), denoted by  ($w_0$ and $w_1$) respectively. These weights depend on both the surrogate function by $S(x)$ and neuron's firing probability $F(h)$ . In particular,  the weight assigned to linear approximation around $o=1$ is an increasing function of $F(h)$ such that as the probability of firing increases, the estimator places more weight on this terminal. This is conceptually aligned with the ST estimator, and contributes to a low-variance gradient estimate, as shown in Proposition 4. 

The dampening factor is maximal when $F(h) = 0.5$ and decreases symmetrically as the firing probability deviates from this point. This behaviour can be interpreted as selectively suppressing unreliable gradients—those arising when the approximation deviates significantly from the balanced trapezoid rule.
%Whilst the primary role of the dampening factor is to reduce variance for all gradients, the greater dampening away from $0.5$ can be interpreted as suppressing the gradient approximations which stray from the balanced trapezoid approximation. 

For the surrogate function,  for a sigmoid with temperature $k$, as this temperature goes to $0$ we have $w_0=w_1=0.5$, which leads to the traditional trapezoid approximation of the finite difference, where equal weight is assigned to each terminal and there is no dampening factor. On the other hand the dampening factor is an decreasing function of the temperature (for a given firing probability), highlighting the role of higher temperatures in variance reduction. As for the influence of the temperature on the weight assigned to each terminal, if $F(h)>0.5$, then the weight assigned to $o=1$ is an increasing function of the temperature, and if $F(h)<0.5$ the weight assigned to $o=1$ is a decreasing function of the temperature. Thus, increasing temperature shifts the approximation to place more weight on the more likely outcome—a strategy that, as formalised in Proposition 4, optimally reduces estimator variance.

\subsubsection*{IW-ST}
We arrive at our novel class of estimators, IW-ST($p$) by allowing the weight assigned to each terminal be chosen more flexibly as some parameter $p$:

\begin{align}
  &\mathcal{L}_1 - \mathcal{L}_0 = 
  \int_{0}^{1} \frac{d\mathcal{L}}{d o}\,d o \approx 
  (1-p)\,\frac{d\mathcal{L}}{d o}\Bigm|_{o=0}
  \;+\;
  p\,\frac{d\mathcal{L}}{d o}\Bigm|_{o=1} = 
  \mathbb{E}_{\,o \sim \mathrm{Bern}(p)}
  \Bigl[\tfrac{d\mathcal{L}}{d o}\Bigr].
\end{align}

Note that we also omit the dampening term; we demonstrate in Appendix E that dampening leads to worse training dynamics in BNNs as it exacerbates vanishing gradients, though it can be of use in SNNs where exploding gradients are more prevalent. Note also that $p$ can be chosen as a function of $h$, enabling the weight assigned to each terminal to be adapted for each neuron. 

In practice, we can use the IW-ST($p$) estimator based on outputs observed during the forward pass via importance sampling, where the final expectation is then approximated with Monte Carlo samples:

\begin{align}
  \mathbb{E}_{\,o\sim\mathrm{Bern}(p)}\Bigl[\tfrac{d\mathcal{L}}{d o}\Bigr]
  &= 
  \mathbb{E}_{\,o\sim\mathrm{Bern}\bigl(F(h)\bigr)}
    \Bigl[\,
      \Bigl(\tfrac{p}{F(h)}\Bigr)^{o}
      \,
      \Bigl(\tfrac{1-p}{1-F(h)}\Bigr)^{1-o}
      \;\tfrac{d\mathcal{L}}{d o}\Bigr]
  \\[0.75ex]
  &\approx
    \Bigl(\tfrac{p}{F(h)}\Bigr)^{o}
    \,
    \Bigl(\tfrac{1-p}{1-F(h)}\Bigr)^{1-o}
    \;\tfrac{d\mathcal{L}}{d o}.
\end{align}

This enables us to use the IW-ST estimator for any choice of $p$ without any changes to the forward pass. As per Remark 3, setting $p=F(h)$ recovers the classical ST estimator.

\section*{C - Bias and Variance of IW-ST}
For the proofs not provided in the main text, we restate the result and provide full derivations here. Note that in following we assume $F(h)$ is the normal CDF with variance $\sigma^2$, $F(h)=\Phi(\frac{h}{\sigma})$, where this variance can differ between neurons. When we condition on 'input' in the following, this refers to the input to the network (ie a given training example for supervised learning). Additionally, $\phi()$ is the standard normal PDF.

\begin{theorem}
    The ST method is unbiased when:
\begin{enumerate}
  \item For all neurons $i$ in any layer $l$, for all configurations of $o^{(l-1)}$,
  \[
    p\bigl(o^{(l)}_i=1 \mid o^{(l-1)},\text{input}\bigr)
    \;=\;
    \sum_k g\bigl(o_k^{(l-1)},\text{input}\bigr),
  \]
  i.e.\ each neuron’s probability of firing is a linear combination of its inputs.

  \item For all output neurons $j$ in layer $L$,
  \[
    p\bigl(o_j^{(L)} =1\mid o^{(L-1)},\text{input}\bigr)
    \;=\;
    \mathbb{E}_{\,o^{(L-1)}}\!\bigl[p\bigl(o_j^{(L)} =1\mid o^{(L-1)},\text{input}\bigr)\bigr].
  \]
  \item The loss function is linear in the output neurons:
  \[
   \mathcal{L}(o_i^{(L)}=1)- \mathcal{L}(o_i^{(L)}=0)=\frac{d\mathcal{L}}{do_i^{(L)}}.
  \]
\end{enumerate}
\end{theorem}

\begin{proof}

The core equation underlying the ST estimator is that the finite difference at layer $l$ can be written as a linear combination of finite differences at layer $l+1$ as follows:

\begin{align}
  &\mathbb{E}\bigl[\mathcal{L}(o_i^{(l)}=1)-\mathcal{L}(o_i^{(l)}=0)\bigr]
  \nonumber\\[-0.5ex]
  &\quad\approx
  \mathbb{E}_{\,o^{(l)}}\Bigl[
    \sum_k 
      \frac{w_{ki}^{(l+1)}}{\sigma_k}
      \,\phi\!\Bigl(\tfrac{h_k^{(l+1)}}{\sigma_k}\Bigr)\,
      \mathbb{E}\bigl[\mathcal{L}(o_k^{(l+1)}=1)
        - \mathcal{L}(o_k^{(l+1)}=0)\,\bigm|\,o^{(l)}\bigr]
  \Bigr].
  \label{eq:ST}
\end{align}

Where in practice the expectations are estimated by Monte Carlo samples (i. e. a forward pass through the network). We can expand the inner expectation on the LHS, summing over all possible configurations of neurons in layer $l+1$:

\begin{align}
  &\mathbb{E}\bigl[\mathcal{L}(o_i^{(l)}=1)-[\mathcal{L}(o_i^{(l)}=0)\bigr]
  \nonumber\\
  &=\;
  \mathbb{E}_{\,o^{(l)}}\Bigl[
    \sum_{\gamma}
      \mathbb{E}\bigl[\mathcal{L} \mid o^{(l+1)}=\gamma\bigr]\,
      \bigl(
        p\bigl(o^{(l+1)}=\gamma \mid o_i^{(l)}=1,\,o_{-i}^{(l)}\bigr)
        -
        p\bigl(o^{(l+1)}=\gamma \mid o_i^{(l)}=0,\,o_{-i}^{(l)}\bigr)
      \bigr)
  \Bigr].
\end{align}

\(\gamma\) is a vector which represents an arbitrary configuration of the outputs of neuron outputs at layer \(l+1\), and we sum over all such configurations. For clarity in what follows, when conditioning on events such as \(o^{(l+1)}=\gamma\), we sometimes write this as \(o^{(l+1)}\), i. e.  \( \mathbb{E}[L|o^{(l+1)}]:= \mathbb{E}[L|o^{(l+1)}=\gamma]\) .  $o^{(l)}_{-i}$ denotes all neurons in layer $l$ except neuron $i$. 

To arrive at the ST estimator, we need to only make two moves from here. Firstly, we require the expectation to decompose as a sum: 

\begin{equation}
  \mathbb{E}\bigl[\mathcal{L} \,\bigm|\, o^{(l+1)}\bigr]
  \;\approx\;
  \sum_{j}
    \mathbb{E}\bigl[\mathcal{L} \,\bigm|\, o_j^{(l+1)},\,o^{(l)}\bigr]
  \;+\; c \label{eq:add_exp}
\end{equation}

for any constant \(c\). Intuitively, this condition ensures the interaction effects between neurons in layer \(l+1\) are small, such that the expected combined effects of these neurons is equal to the sum of their expected individual effects. Secondly, we require that the finite difference in probabilities can be written via linear approximation:

\begin{align}
&p\bigl(o_k^{(l+1)}=1 \mid o_i^{(l)}=1,\,o_{-i}^{(l)}\bigr)
 - p\bigl(o_k^{(l+1)}=1 \mid o_i^{(l)}=0,\,o_{-i}^{(l)}\bigr) \\
&\quad=
  F\Bigl(\sum_{j \neq i} w_{kj}^{(l+1)}\,o_j^{(l)} + w_{ki}^{(l+1)}\Bigr)
  - F\Bigl(\sum_{j \neq i} w_{kj}^{(l+1)}\,o_j^{(l)}\Bigr) \\
&\quad\approx
  F'\!\Bigl(\sum_j w_{kj}^{(l+1)}\,o_j^{(l)}\Bigr)\;w_{ki}^{(l+1)} \\
&\quad=
  \phi\!\Bigl(\tfrac{h_k^{(l+1)}}{\sigma_k}\Bigr)\,
  \frac{w_{ki}^{(l+1)}}{\sigma_k}
\end{align}

This second approximation is relatively weak - requiring only that the effect of a single input neuron toggling on the probability of its connected neuron firing is approximately linear, which will hold when the \(\frac{w^{(l+1)}_{ki}}{\sigma_k}\) is sufficiently small. Note that in the above expression, we assume the expectation is over $o_i^{(l)} \sim Bernoulli(F(h_i^{(l)}))$, which corresponds to the classical ST estimator. We could instead use the IW-ST($p$) estimator, which uses $o_i^{(l)} \sim Bernoulli(p)$ instead, with the implications of this choice analysed in Propositions 3 and 4. 

The additive expectation assumption \ref{eq:add_exp} on the other hand is strong. To understand this condition, we can write \(\mathbb{E}[L|o^{(l+1)}]\) as

\begin{equation}
  \mathbb{E}\bigl[L \,\bigm|\, o^{(l+1)}\bigr]
  \;=\;
  \sum_{\gamma^{(l+2)}}
    \mathbb{E}\bigl[L \,\bigm|\, o^{(l+2)}=\gamma^{(l+2)}\bigr]
    \;p\bigl(o^{(l+2)}=\gamma^{(l+2)} \,\bigm|\, o^{(l+1)}\bigr)
\end{equation}

For this to be equal to \(\sum_j \mathbb{E}[L|o_j^{(l+1)},o^{(l)}]+c\), since the expression depends on \(o^{(l+1)}\) through the probability term, it is necessary and sufficient that  

\begin{equation}
  p\bigl(o^{(l+2)} = \gamma^{(l+2)} \,\bigm|\, o^{(l+1)}\bigr)
  \;=\;
  \sum_{k} 
    g_{k}\!\bigl(o_{k}^{(l+1)},\,o^{(l)}\bigr)
\end{equation}

However, it is difficult to implement this assumption in its current form, as the left hand side is a \textit{joint} probability statement, making its decomposition into a sum more difficult than a single variable case. We can however propose a set of sufficient conditions, which when satisfied, guarantee this linear decomposition of the joint distribution. These conditions are:

\begin{align}
  p\bigl(o_i^{(l+2)}=1 \mid o^{(l+1)}\bigr)
  &\approx
  \sum_{k}
    g_{i,k}\!\bigl(o_k^{(l+1)},\,o^{(l)}\bigr), \\[1ex]
  \mathbb{E}\bigl[L \,\bigm|\, o^{(l+2)}\bigr]
  &\approx
  \sum_{j}
    \mathbb{E}\!\bigl[L \,\bigm|\, o_j^{(l+2)},\,o^{(l)}\bigr]
  \;+\; c
\end{align}

The first assumption is that each \textit{marginal} probability for neurons in layer $l+2$ can be decomposed linearly. The second condition on the other hand is equivalent to the condition we are trying to enforce, but using neurons in one layer downstream ($l+2$ instead of $l+1$).  We can therefore form a set of inductive assumptions, where the assumption is satisfied in layer \(s\) only if it is satisfied in layer \(s+1\). Mathematically, for any \(s>l\), we have

\begin{align}
  \mathbb{E}\bigl[L \,\bigm|\, o^{(s)}\bigr]
  &\approx
  \sum_{j}
    \mathbb{E}\bigl[L \,\bigm|\, o_j^{(s)},\,o^{(l)}\bigr]
  \;+\; c,
  \\[0.75ex]
  \intertext{only if}
  \mathbb{E}\bigl[L \,\bigm|\, o^{(s+1)}\bigr]
  &\approx
  \sum_{j}
    \mathbb{E}\bigl[L \,\bigm|\, o_j^{(s+1)},\,o^{(l)}\bigr]
  \;+\; c,
  \\[0.75ex]
  \intertext{and}
  p\bigl(o_i^{(s+1)}=1 \mid o^{(s)}\bigr)
  &\approx
  \sum_{k}
    g_{i,k}\!\bigl(o_k^{(s)},\,o^{(l)}\bigr).
\end{align}

Therefore, to show that equation \ref{eq:ST}  holds for neurons in layer $l$, by unrolling these assumptions, we arrive at two conditions:

\begin{equation}
  p\bigl(o_i^{(s)}=1 \mid o^{(s-1)}\bigr)
  \;=\;
  \sum_{k}
    g\!\bigl(o_k^{(s-1)},\,o^{(l)}\bigr),
  \quad
  \text{for all }s>l.
\end{equation}

\begin{align}
  \mathbb{E}\bigl[L \,\bigm|\, o^{(L)}\bigr]
  &= L\bigl(o^{(L)}\bigr)\approx
  \sum_{j}
    \mathbb{E}\!\bigl[L \,\bigm|\, o_j^{(L)},\,o^{(l)}\bigr]
  \;+\; c.
\end{align}

For the second condition here, note that for most common loss function used to train neural networks (e. g. MSE, cross entropy), the loss function can be decomposed as  \(L(o^{(L)})= \sum_k g_k(o^{(L)}_k)\),  that is, the loss is calculated for each output neuron separately and then the total loss is given by the sum of these individual losses. If we assume this holds, we can see that 

\begin{align}
  \sum_{j}\mathbb{E}\bigl[L \,\bigm|\, o_j^{(L)},\,o^{(l)}\bigr]
  &= 
  \sum_{j}\sum_{k}
    \mathbb{E}\bigl[g_k\bigl(o_k^{(L)}\bigr)
      \,\bigm|\, o_j^{(L)},\,o^{(l)}\bigr]
  \\[0.5ex]
  &= 
  \sum_{j}g_j\bigl(o_j^{(L)}\bigr)
  \;+\;
  \sum_{j}\sum_{k \neq j}
    \mathbb{E}\bigl[g_k\bigl(o_k^{(L)}\bigr)
      \,\bigm|\, o_j^{(L)},\,o^{(l)}\bigr]
  \\[0.5ex]
  &=
  \mathbb{E}\bigl[L \,\bigm|\, o^{(L)}\bigr]
  \;+\;
  \sum_{j}\sum_{k \neq j}
    \mathbb{E}\bigl[g_k\bigl(o_k^{(L)}\bigr)
      \,\bigm|\, o_j^{(L)},\,o^{(l)}\bigr].
\end{align}

Thus the requirement that  \(\mathbb{E}[L|o^{(L)}] \approx \sum_j\mathbb{E}[L|o^{(L)}_j,o^{(l)}]+c\) only holds if \(\sum_{j}\sum_{k \neq j} \mathbb{E}[g_k(o_k^{(L)})|o^{(L)}_j,o^{(l)}]\) is a constant.  However, this is only true if \(o_k^{(L)}\) is conditionally independent of \(o_j^{(L)}\) for all \(k \neq j \), where the conditioning is on \(o^{(l)}\). Assuming that there is some noise present in any of the layers \(l+1:L-1\), this will not be hold, because such noise will induce a dependence between \(o_k^{(L)}\) and \(o_j^{(L)}\). To better understand this dependence, we can expand the second term as 

\begin{multline}
  \mathbb{E}\bigl[g_k(o_k^{(L)}) \,\bigm|\, o_j^{(L)},\,o^{(l)}\bigr]
  = \sum_{\gamma}
      \mathbb{E}\bigl[g_k(o_k^{(L)}) \,\bigm|\, o^{(L-1)}=\gamma,\,o^{(l)}\bigr]
      \;p\bigl(o^{(L-1)}=\gamma \,\bigm|\, o_j^{(L)},\,o^{(l)}\bigr)
  \\[0.75ex]
  = \sum_{\gamma}
      \mathbb{E}\bigl[g_k(o_k^{(L)}) \,\bigm|\, o^{(L-1)}=\gamma\bigr]
      \;\frac{p\bigl(o_j^{(L)} \,\bigm|\, o^{(L-1)}=\gamma\bigr)}
           {\mathbb{E}\!\bigl[p\bigl(o_j^{(L)}\,\bigm|\, o^{(L-1)}\bigr)
            \,\bigm|\, o^{(l)}\bigr]}
      \;p\bigl(o^{(L-1)}=\gamma \,\bigm|\, o^{(l)}\bigr).
\end{multline}

We can see that this term will only be constant when 

\begin{equation}
  p\bigl(o_j^{(L)}=1 \,\bigm|\, o^{(L-1)}\bigr)
  \;=\;
  \mathbb{E}\Bigl[
    p\bigl(o_j^{(L)} =1\,\bigm|\, o^{(L-1)}\bigr)
    \,\bigm|\,
    o^{(l)}
  \Bigr].
\end{equation}

Therefore, equation \ref{eq:ST}  is exact for layer $l$ when the following two conditions hold:

\begin{align}
  p\bigl(o_i^{(s)}=1 \mid o^{(s-1)}\bigr)
  &= 
  \sum_{k} 
    g\!\bigl(o_k^{(s-1)},\,o^{(l)}\bigr),
  \quad \text{for all } s > l,
  \\[0.75ex]
  p\bigl(o_j^{(L)} = 1 \mid o^{(L-1)}\bigr)
  &=
  \mathbb{E}\Bigl[
    p\bigl(o_j^{(L)} = 1 \mid o^{(L-1)}\bigr)
  \Bigm|\,o^{(l)}\Bigr].
\end{align}

Note that so far we have considered when the ST estimator is unbiased for calculating the gradient with respect to neurons in layer $l$.  Importantly, when these conditions hold for a given layer $l$, they are guaranteed to hold for all downstream layers $s>l$, because the structure of the network results in $o^{(t)} \perp o^{(l)}|o^{(s)}$ for $l<s<t$, or in words, any information contained in $o^{(l)}$ about $o^{(t)}$ will also be contained in $o^{(s)}$ for $l<s<t$. We then arrive at the final conditions (1) and (2) stated in Theorem 1, where the conditioning is only on the input to the network,  which will guarantee equation \ref{eq:ST} is valid for all $l$. 

All that remains to arrive at the classical ST estimator is to assume that the finite differences for each of the neurons in the final layer are equal to their linear approximation, or in other words, the loss function is linear in the output neurons (condition 3). Whilst this does not hold exactly in most cases, for common loss functions used in neural network training such as cross-entropy or MSE loss, this linear approximation is reasonable. 

\end{proof}

\begin{corollary}
Minimising the bias corresponds to maximising
\[
  p\!\Bigl(\bigcap_i\bigl\lvert\tfrac{h_i}{\sigma_i}-\mathbb{E}[\tfrac{h_i}{\sigma_i}]\bigr\rvert<\epsilon \;\Bigm|\;\text{input}\Bigr),
\]
for sufficiently small $\epsilon>0$, where $i$ ranges over all neurons.
\end{corollary}

\begin{proof}
Note that we have 

\begin{equation}
  p\bigl(o_i^{(l+1)} = 1 \mid o^{(l)}\bigr)
  = F\bigl(h_i^{(l+1)}\bigr)
  = \Phi\!\Bigl(\frac{h_i^{(l+1)}}{\sigma_i}\Bigr)
\end{equation}

where $h_i^{(l+1)}=\sum_j w^{(l+1)}_{ij} o_j^{(l)}$. Consider a linear approximation of this Gaussian CDF in the terms $o^{(l)}$, around $\mathbb{E}[\tfrac{h_i^{(l+1)}}{\sigma_i}]$:

\begin{equation}
  p\bigl(o_i^{(l+1)}=1\mid o^{(l)}\bigr)
  \approx
  \Phi\!\Bigl(\mathbb{E}\bigl[\tfrac{h_i^{(l+1)}}{\sigma_i}\bigr]\Bigr)
  +\phi\!\Bigl(\mathbb{E}\bigl[\tfrac{h_i^{(l+1)}}{\sigma_i}\bigr]\Bigr)
  \Bigl(\tfrac{\sum_j w_{ij}^{(l+1)}\,o_j^{(l)}}{\sigma_i}
    -\mathbb{E}\bigl[\tfrac{h_i^{(l+1)}}{\sigma_i}\bigr]\Bigr).
\end{equation}

Since the Gaussian CDF is continuous, its linear approximation will only be accurate in a small interval around $\mathbb{E}[\tfrac{h_i^{(l+1)}}{\sigma_i}]$:

\begin{equation}
  \left\lvert
    \frac{h_i^{(l+1)}}{\sigma_i}
    - 
    \mathbb{E}\!\Bigl[\frac{h_i^{(l+1)}}{\sigma_i}\Bigr]
  \right\rvert
  < \epsilon.
\end{equation}

As the probability of this event converges to $1$, for sufficiently small $\epsilon$ the linear approximation becomes exact, condition (1) from Theorem 1 is satisfied, and the bias of the ST estimator is minimised. 

To enforce condition (2) from Theorem 1, we wish to set:
\begin{equation}
  \left\lvert
    p\bigl(o_j^{(L)} = 1 \mid o^{(L-1)}\bigr)
    \;-\;
    \mathbb{E}\!\Bigl[
      p\bigl(o_j^{(L)} = 1 \mid o^{(L-1)}\bigr)
    \Bigr]
  \right\rvert
  < \epsilon.
\end{equation}

If the neurons in the output layer are binary, then using the fact that \(p\bigl(o_j^{(L)}=1\mid o^{(L-1)}\bigr)  = \Phi\!\bigl(\tfrac{h_j^{(L)}}{\sigma_j}\bigr)\), since $\Phi$ is $L$-Lipschitz, we can enforce this bound on the probabilities by enforcing a bound on the membrane potentials:

\begin{align}
  \left\lvert
    \frac{h_j^{(L)}}{\sigma_j}
    - 
    \mathbb{E}\!\Bigl[\frac{h_j^{(L)}}{\sigma_j}\Bigr]
  \right\rvert
  &< \epsilon
  \quad\implies\quad
  \left\lvert
    \Phi\!\Bigl(\tfrac{h_j^{(L)}}{\sigma_j}\Bigr)
    - 
    \Phi\!\Bigl(\mathbb{E}\!\bigl[\tfrac{h_j^{(L)}}{\sigma_j}\bigr]\Bigr)
  \right\rvert
  < L\,\epsilon
  \nonumber\\[0.75ex]
  &\quad\implies\quad
  \left\lvert
    \Phi\!\Bigl(\tfrac{h_j^{(L)}}{\sigma_j}\Bigr)
    - 
    \mathbb{E}\!\Bigl[\Phi\!\bigl(\tfrac{h_j^{(L)}}{\sigma_j}\bigr)\Bigr]
  \right\rvert
  < 2\,L\,\epsilon,
\end{align}

where the first statement follows directly from the definition of $L$-Lipschitz, and the second follows from an application of the triangle inequality. Therefore, maximising 

\begin{equation}
  \Pr\Bigl(
    \bigcap_{i}
      \Bigl\lvert
        \frac{h_{i}}{\sigma_{i}}
        - 
        \mathbb{E}\!\Bigl[\frac{h_{i}}{\sigma_{i}}\Bigr]
      \Bigr\rvert
      < \epsilon
    \;\Bigm|\;
    \text{input}
  \Bigr).
\end{equation}

Minimises the bias via both Condition 1 and 2 from Theorem 1. 

\end{proof}

\setcounter{proposition}{2}
\begin{proposition}
When zero–mean and zero-variance penalties are applied for each neuron, the bias of IW-ST($p$) can be reduced by setting:
\[
  p \in
  \begin{cases}
    [\,0.5,\,1\,] &\text{if }F(h)>0.5,\\
    [\,0,\,0.5\,] &\text{if }F(h)<0.5.
  \end{cases}
\]
\end{proposition}
\begin{proof}

From the proof of Theorem 1, we saw that the bias of the ST estimator is in part due to the linear approximation:

\begin{equation}
\begin{split}
  &F\Bigl(\sum_{j\neq i} w_{kj}o_j^{(l)} + w_{ki}\Bigr)
   - F\Bigl(\sum_{j\neq i} w_{kj}o_j^{(l)}\Bigr)\\
  &\quad\approx
   \mathbb{E}_{\,o_i^{(l)}\sim\mathrm{Bernoulli}(p)}\!
   \Bigl[F'\bigl(\sum_{j} w_{kj}o_j^{(l)}\bigr)\Bigr]
   \,w_{ki}\\
  &\quad=
   p\,F'\!\Bigl(\sum_{j} w_{kj}o_j^{(l)}\Bigr)\,w_{ki}
   +(1-p)\,F'\!\Bigl(\sum_{j\neq i} w_{kj}o_j^{(l)}\Bigr)\,w_{ki}.
\end{split}
\end{equation}

The IW-ST($p$) estimates the expectation using $o_i^{(l)} \sim Bernoulli(p)$, where $p=F(h)$ recovers the classical ST estimator. The bias of the IW-ST estimator for a given choice of $p$ can then be understood as how it impacts the linear approximation of \(F(\sum_{j \neq i} w_{kj} o_j^{(l)}+w_{ki}) -F(\sum_{j \neq i} w_{kj} o_j^{(l)})\). Suppose that $F()$ is the standard Gaussian CDF $\Phi$ (where we have set the variance to 1 without lack of generalisation) with derivative $\phi$.  Additionally, we assume that when $0$-mean and $0$-variance penalties are enforced, when $F(h_i)>0.5$, $|\sum_{j} w_{kj} o_j^{(l)}|<|\sum_{j \neq i} w_{kj} o_j^{(l)}|<1$ (the outcome with higher probability gives a sum closer to $0$ than the lower probability outcome). The assumption formalises the intuition that the constraint on $p(|\sum_{j} w_{kj} o_j^{(l)}|>\epsilon)$ for progressively smaller $\epsilon$ forces \(|\sum_{j} w_{kj} o_j^{(l)}|\) to be small for high-probability configurations of $o^{(l)}$. Therefore, if a neuron is more likely to fire than not, we expect the resulting membrane potential based on the neuron firing to be smaller than if it had not fired. We additionally constrain this sum to be below $1$ which corresponds to setting $\epsilon=1$, as the normal PDF is concave in this region which is required for our proof. 

For clarity, let $x=\sum_{j \neq i} w_{kj} o_j^{(l)}$, $h=w_{ki}$. The optimal $p$ to minimise bias satisfies:

\begin{align}
  h\bigl(p\,\phi(x+h) + (1 - p)\,\phi(x )\bigr)
  &= \int_{x}^{x+h} \phi(t)\,\mathrm{d}t, \\[1ex]
  p^*
  &= \frac{\displaystyle\frac{1}{h}\int_{x}^{x+h}\phi(t)\,\mathrm{d}t \;-\;\phi(x )}
          {\phi(x+h)\;-\;\phi(x )}.
\end{align}

Note that since $\phi$ is concave in $[-1,1]$, so for $h>0$:

\begin{equation}
  \frac{1}{h}\int_{x}^{x+h}\phi(t)\,\mathrm{d}t
  \;>\;
  \tfrac{1}{2}\bigl(\phi(x) + \phi(x+h)\bigr)
\end{equation}

Therefore, if $\phi(x+h)>\phi(x)$, we have:

\begin{equation}
  p^*
  \;>\;
  \frac{\displaystyle \tfrac{1}{2}\bigl(\phi(x+h) + \phi(x)\bigr) - \phi(x)}
       {\displaystyle \phi(x+h) - \phi(x)}
  \;=\;
  \tfrac{1}{2}.
\end{equation}

And if $\phi(x+h)<\phi(x)$, we have:

\begin{equation}
  p^*
  \;<\;
  \frac{\displaystyle \tfrac{1}{2}\bigl(\phi(x+h) + \phi(x)\bigr) - \phi(x)}
       {\displaystyle \phi(x+h) - \phi(x)}
  \;=\;
  \tfrac{1}{2}.
\end{equation}

In conclusion we have

\begin{align}
  F\bigl(h_i\bigr) > 0.5
  &\;\implies\;
  \bigl\lvert \sum_{j} w_{kj}\,o_j^{(l)} \bigr\rvert
  < 
  \bigl\lvert \sum_{j \neq i} w_{kj}\,o_j^{(l)} \bigr\rvert
  \\[0.5ex]
  &\;\implies\;
  \phi\!\Bigl(\sum_{j} w_{kj}\,o_j^{(l)}\Bigr)
  >
  \phi\!\Bigl(\sum_{j \neq i} w_{kj}\,o_j^{(l)}\Bigr)
  \\[0.5ex]
  &\;\implies\;
  p^* > 0.5.
\end{align}

where the first step is by assumption. The opposite argument applies for $F(h_i)<0.5$.

\end{proof}

\begin{proposition}
Under mild assumptions on the loss, and assuming that $F(h)$ is optimised to minimise this loss, the variance of IW-ST($p$) is minimised when
\[
  p \in
  \begin{cases}
    [\,F(h),\,1\,] &\text{if }F(h)>0.5,\\
    [\,0,\,F(h)\,] &\text{if }F(h)<0.5,
  \end{cases}
\]
\end{proposition}

\begin{proof}
The conditional variance for a given choice of $p$, as a function of $h$, can be written as 

\begin{equation}
  \operatorname{Var}\bigl[\text{IW-ST}(p)\mid h\bigr]
  = 
  \frac{(1-p)^2\,(\mathcal{L}_0')^2}{1 - F(h)}
  \;+\;
  \frac{p^2\,(\mathcal{L}_1')^2}{F(h)}
  \;-\;
  \bigl((1-p)\,\mathcal{L}_0' + p\,\mathcal{L}_1'\bigr)^{2},
\end{equation}

where $\mathcal{L}'_0=\frac{d\mathcal{L}}{do}|_{o=0}$ and $\mathcal{L}'_1=\frac{d\mathcal{L}}{do}|_{o=1}$

This variance is minimised when 

\begin{equation}
  p^{*}
  \;=\;
  \frac{\mathcal{L}_0'\,F(h)}
       {\mathcal{L}_0'\,F(h)\;+\;\mathcal{L}_1'\,(1 - F(h))}.
\end{equation}

We can observe the following relationship between $\mathcal{L}_1'
$, $\mathcal{L}_0'$, and $p^*$:

\begin{align}
  |\mathcal{L}_0'| = |\mathcal{L}_1'|
  &\;\implies\;
  p^* = F(h), \\[0.5ex]
  |\mathcal{L}_0'| < |\mathcal{L}_1'|
  &\;\implies\;
  p^* \in [\,0,\,F(h)\,], \\[0.5ex]
  |\mathcal{L}_0'| > |\mathcal{L}_1'|
  &\;\implies\;
  p^* \in [\,F(h),\,1\,].
\end{align}

 If we assume that the network has undergone some training such that the loss as a function of a given neuron is convex on $[0,1]$ and $sign(\mathcal{L}_1')=sign(\mathcal{L}_0')$ on this interval (i. e. the network is near a local minima), then we have $\mathcal{L}_0<\mathcal{L}_1 \implies |\mathcal{L}_0'| < |\mathcal{L}_1'| $ and $\mathcal{L}_0>\mathcal{L}_1 \implies |\mathcal{L}_0'| > |\mathcal{L}_1'| $. 
 
Note also that the weights are updated by gradient descent using $(\mathcal{\mathcal{L}}_1 - \mathcal{\mathcal{L}}_0) \cdot\frac{dF(h)}{dw}$, such that when $\mathcal{L}_1<\mathcal{L}_0$, the weights are updated so as to increase the probability of a neuron firing and the opposite is true for when $\mathcal{L}_1>\mathcal{L}_0$. In other words,  if $F(h)>0.5$, this implies that the network has learnt that on average $\mathcal{L}_1<\mathcal{L}_0$. Using this, we arrive at the result:

\begin{align}
  F(h) > 0.5
  &\;\iff\;
  \mathcal{L}_1 < \mathcal{L}_0
  \;\implies\;
  \bigl\lvert\mathcal{L}_1'\bigr\rvert < \bigl\lvert\mathcal{L}_0'\bigr\rvert
  \;\implies\;
  p^* \in [\,F(h),\,1], \\[0.75ex]
  F(h) < 0.5
  &\;\iff\;
  \mathcal{L}_0 < \mathcal{L}_1
  \;\implies\;
  \bigl\lvert\mathcal{L}_0'\bigr\rvert < \bigl\lvert\mathcal{L}_1'\bigr\rvert
  \;\implies\;
  p^* \in [\,0,\,F(h)].
\end{align}

\end{proof}

\section*{D - Local Reparameterisation for CNNs and SNNs}

To use the IW-ST estimator, we must apply the local reparameterisation trick to the noisy weights so that each neuron's output can be treated as a Bernoulli random variable, with a probability that is a differentiable function of the network parameters. Two conditions ensure this: (i) the preactivation is a linear function of the weights, and (ii) the activation function is applied directly to this linear combination.

\subsubsection*{Applied to convolutional networks}
For a convolutional network, we can write the output of a neuron in a convolutional layer as $o_c^{(l)}=H(W_c^{(l)}*o^{(l-1)})$, where $c$ refers to the channel dimension. Since the convolution is a linear operator, we can see this satisfies i) and ii) above, making the Bayesian approach easy to apply in this context. To handle weight-sharing, when applying the convolution, we can either sample the weights $W_c^{(l)}$ once, and apply this same sampled weight at all positions (meaning the randomness is shared for all weights), or we could sample a separate weight $W_c^{(l,p)}$ for each position, giving $o_c^{(l)}=H(W_c^{(l,p)}*o^{(l-1)})$, where each weight is sampled from the same distribution, $W_c^{(l,p)} \sim N(M_c^{(l)},\sigma_c^{(l)2})$. The latter approach means that the weight parameters are shared, but the randomness is independent for each position. The former approach is problematic when attempting to use the ST estimator, as using the same sample at each position induces a correlation between the output neurons, meaning the noise CDF for each neuron will be complex to work with. Sampling independent weights at each position on the other hand means each neuron will be independent, leading to a straight-forward application of the ST estimator.

\subsubsection*{Applied to spiking neural networks}
Recall that we assume an SNN is governed by the following equations:
\begin{align}
  o_{i,t}^{(l)}
  &= 
  H\bigl(h_{i,t}^{(l)} - \theta\bigr), \\[1ex]
  h_{i,t}^{(l)}
  &= 
  \beta\,h_{i,t-1}^{(l)}
  \;+\;
  \sum_{j} w_{ij}^{(l)}\,o_{j,t}^{(l-1)}
  \;+\;
  \sum_{k \neq i} r_{ik}^{(l)}\,o_{k,t-1}^{(l)}
  \;-\;
  \theta\,o_{i,t-1}^{(l)}.
\end{align}

Crucially, we can see that both the feedforward weights $w_{ij}^{(l)}$ and recurrent weights $r_{ik}^{(l)}$ enter the membrane potential linearly, and the output is the Heaviside of the membrane potential, so i) and ii) are satisfied. For weight sharing, we apply the same argument as we made in the case of CNNs - that is, we prefer to sample independent weights at each time step from the same underlying distribution -  $w_{ij}^{(l,t)} \sim N(m_{ij}^{(l)},\sigma_{ij}^{2(l-1)})$, meaning we are sharing the underlying parameters but not sharing the randomness across time steps. Note however that we still have an issue -- the persistence of the membrane potential induces a dependence between $o^{(l)}_{i,t}$ and $o^{(l)}_{i,s}$ for all times $t,s$, which again makes the CDF for each neuron difficult to work with. The solution to this challenge relies upon using an alternative view of the SNN, which we call the convolutional view: by expanding the recurrence relation above, we can write the membrane potential as

\begin{equation}
  h_{i,t}^{(l)}
  = 
  \sum_{s=0}^{t-1}
    \beta^{s}\,
    \Bigl(
      \sum_{j} w_{ij}^{(l,t,s)}\,o_{j,t-s}^{(l-1)}
      \;+\;
      \sum_{k \neq i} r_{ik}^{(l,t,s)}\,o_{k,t-s-1}^{(l)}
      \;-\;
      \theta\,o_{i,t-s-1}^{(l)}
    \Bigr)
\end{equation}

This approach then involves not just sampling separate independent weights at every new time step, but at each time step resampling the weights from previous time steps, $w_{ij}^{(l,t,s)} \sim N(m_{ij}^{(l)},\sigma_{ij}^{2(l-1)})$. This ensures that the output at each time step is independent, making the CDF easy to derive for each neuron at each time step. Importantly, we can still implement this scheme efficiently using the recurrent view, without the need of actually resampling the weights for all previous time steps. Applying the reparameterisation $w_{ij}^{(l,t,s)} =m_{ij}^{(l)}+\sigma_{ij}^{(l)}\epsilon^{(f,l,t,s)}_{ij}$ and  $r_{ik}^{(l,t,s)} =p_{ik}^{(l)}+\nu_{ik}^{(l)}\epsilon^{(r,l,t,s)}_{ik}$, where $f$ and $r$ denote noise relating to the feedforward and recurrent weights respectively, and $\epsilon$ is used to denote a sample from a standard normal random variable, we can write the membrane potential as:

\begin{align}
  h_{i,t}^{(l)}
  &= 
  \sum_{s=0}^{t-1} \beta^s
    \Bigl(
      \sum_{j} m_{ij}^{(l)}\,o_{j,t-s}^{(l-1)}
      + 
      \sum_{k \neq i} p_{ik}^{(l)}\,o_{k,t-s-1}^{(l)}
      - 
      \theta\,o_{i,t-s-1}^{(l)}
    \Bigr) 
  \nonumber\\
  &\quad+
  \sum_{s=0}^{t-1} \beta^s
    \Bigl(
      \sum_{j} \sigma_{ij}^{(l)}\,\epsilon_{ij}^{(f,l,t,s)}\,o_{j,t-s}^{(l-1)}
      + 
      \sum_{k \neq i} \nu_{ik}^{(l)}\,\epsilon_{ik}^{(r,l,t,s)}\,o_{k,t-s-1}^{(l)}
    \Bigr)\!.
\end{align}

Let use consider the first term:

\begin{equation}
  h^{*(l)}_{i,t}
  = 
  \sum_{s=0}^{t-1} \beta^s
  \Bigl(
    \sum_{j} m_{ij}^{(l)}\,o_{j,t-s}^{(l-1)}
    \;+\;
    \sum_{k \neq i} p_{ik}^{(l)}\,o_{k,t-s-1}^{(l)}
    \;-\;
    \theta\,o_{i,t-s-1}^{(l)}
  \Bigr).
\end{equation}

Note that we can write this sum as a recursive expression:

\begin{align}
  h^{*(l)}_{i,t}
  &= 
  \beta \sum_{s=1}^{t-1} \beta^{s-1}
    \Bigl(
      \sum_{j} m_{ij}^{(l)}\,o_{j,t-s}^{(l-1)}
      + 
      \sum_{k \neq i} p_{ik}^{(l)}\,o_{k,t-s-1}^{(l)}
      - 
      \theta\,o_{i,t-s-1}^{(l)}
    \Bigr) \nonumber\\
  &\quad
  + 
  \sum_{j} m_{ij}^{(l)}\,o_{j,t}^{(l-1)}
  + 
  \sum_{k \neq i} p_{ik}^{(l)}\,o_{k,t-1}^{(l)}
  - 
  \theta\,o_{i,t-1}^{(l)}, \\[1ex]
  &= 
  \beta\,h^{*(l)}_{i,t-1}
  + 
  \sum_{j} m_{ij}^{(l)}\,o_{j,t}^{(l-1)}
  + 
  \sum_{k \neq i} p_{ik}^{(l)}\,o_{k,t-1}^{(l)}
  - 
  \theta\,o_{i,t-1}^{(l)}.
\end{align}

This gives us the evolution of the deterministic membrane potential from the recurrent perspective. 

For the second term of the sum, we first observe that this is a sum of independent normal random variables, and therefore itself is a normal random variable, characterised by its mean and variance. To calculate its mean and variance, we have:

\begin{equation}
\begin{split}
  &\mathbb{E}\Bigl[\sum_{s=0}^{t-1}\beta^s\bigl(
      \sum_{j}\sigma_{ij}^{(l)}\,\epsilon_{ij}^{(f,l,t,s)}\,o_{j,t-s}^{(l-1)}
      \;+\;
      \sum_{k\neq i}\nu_{ik}^{(l)}\,\epsilon_{ik}^{(r,l,t,s)}\,o_{k,t-s-1}^{(l)}
    \bigr)\Bigr]\\
  &=\sum_{s=0}^{t-1}\beta^s\bigl(
      \sum_{j}\sigma_{ij}^{(l)}\,\mathbb{E}[\epsilon_{ij}^{(f,l,t,s)}]\,o_{j,t-s}^{(l-1)}
      \;+\;
      \sum_{k\neq i}\nu_{ik}^{(l)}\,\mathbb{E}[\epsilon_{ik}^{(r,l,t,s)}]\,o_{k,t-s-1}^{(l)}
    \bigr)\\
  &=0.
\end{split}
\end{equation}

\begin{align}
  \mathrm{Var}&\Biggl[\sum_{s=0}^{t-1}\beta^s
    \Bigl(
      \sum_{j} \sigma_{ij}^{(l)}\,\epsilon_{ij}^{(f,l,t,s)}\,o_{j,t-s}^{(l-1)}
      \;+\;
      \sum_{k\neq i} \nu_{ik}^{(l)}\,\epsilon_{ik}^{(r,l,t,s)}\,o_{k,t-s-1}^{(l)}
    \Bigr)\Biggr]
  \nonumber\\
  &= 
  \sum_{s=0}^{t-1}\beta^{2s}
    \Bigl(
      \sum_{j} \mathrm{Var}\bigl(\epsilon_{ij}^{(f,l,t,s)}\bigr)\,\sigma_{ij}^{(l)2}\,o_{j,t-s}^{(l-1)2}
      \;+\;
      \sum_{k\neq i} \mathrm{Var}\bigl(\epsilon_{ik}^{(r,l,t,s)}\bigr)\,\nu_{ik}^{(l)2}\,o_{k,t-s-1}^{(l)2}
    \Bigr)
  \nonumber\\
  &= 
  \sum_{s=0}^{t-1}\beta^{2s}
    \Bigl(
      \sum_{j} \sigma_{ij}^{(l)2}\,o_{j,t-s}^{(l-1)2}
      \;+\;
      \sum_{k\neq i} \nu_{ik}^{(l)2}\,o_{k,t-s-1}^{(l)2}
    \Bigr).
\end{align}

Let \( \kappa^{(l)2}_{i,t}=\sum_s\beta^{2s}(\sum_j \sigma_{ij}^{(l)2}o^{(l-1)2}_{j,t-s}+\sum_{k \neq i} \nu_{ik}^{(l)2} o^{(l)2}_{k,t-s-1}))\). This variance itself follows a recurrence relation:

\begin{align}
  \kappa_{i,t}^{(l)2}
  &= 
  \beta^2 \sum_{s=1}^{t-1} \beta^{2(s-1)}
     \Bigl(
       \sum_{j} \sigma_{ij}^{(l)2}\,o_{j,t-s}^{(l-1)2}
       \;+\;
       \sum_{k\neq i} \nu_{ik}^{(l)2}\,o_{k,t-s-1}^{(l)2}
     \Bigr)
  \nonumber\\
  &\quad
  +\;\sum_{j} \sigma_{ij}^{(l)2}\,o_{j,t}^{(l-1)2}
  \;+\;
  \sum_{k\neq i} \nu_{ik}^{(l)2}\,o_{k,t-1}^{(l)2}
  \\[0.75ex]
  &= 
  \beta^2\,\kappa_{i,t-1}^{(l)2}
  \;+\;
  \sum_{j} \sigma_{ij}^{(l)2}\,o_{j,t}^{(l-1)2}
  \;+\;
  \sum_{k\neq i} \nu_{ik}^{(l)2}\,o_{k,t-1}^{(l)2}.
\end{align}

We therefore have the following equations which describe how the membrane potential and variance evolves across time:

\begin{align}
  h^{(l)}_{i,t}
  &= 
  h^{*(l)}_{i,t}
  \;+\;
  \kappa^{(l)}_{i,t}\,\epsilon^{(l)}_{i,t}, \\[0.75ex]
  h^{*(l)}_{i,t}
  &= 
  \beta\,h^{*(l)}_{i,t-1}
  \;+\;
  \sum_{j} m_{ij}^{(l)}\,o_{j,t}^{(l-1)}
  \;+\;
  \sum_{k\neq i} p_{ik}^{(l)}\,o_{k,t-1}^{(l)}
  \;-\;
  \theta\,o_{i,t-1}^{(l)}, \\[0.75ex]
  \kappa^{(l)2}_{i,t}
  &= 
  \beta^2\,\kappa^{(l)2}_{i,t-1}
  \;+\;
  \sum_{j} \sigma_{ij}^{(l)2}\,o_{j,t}^{(l-1)2}
  \;+\;
  \sum_{k\neq i} \nu_{ik}^{(l)2}\,o_{k,t-1}^{(l)2},
\end{align}

where $h^{*(l)}_{i,t}$ is the "noiseless" membrane potential, to which we add a normal random variable with variance $\kappa_{i,t}^{(l)2}$ to arrive at the membrane potential $h_{i,t}^{(l)}$. We can see that both the "noiseless" membrane potential $h^{*(l)}_{i,t}$ and the variance of the noise $\kappa^{(l)2}_{i,t}$ are governed by recurrence relations, meaning the current value depends only on the previous value and new inputs at the current time step. This means given the variance and noiseless membrane potential at time $t-1$, we calculate the value of the variance and noiseless membrane potential at time $t$ using their recurrence relations,  and then calculate the membrane potential using a new independently sampled noise term $\epsilon^{(l)}_{i,t}$ from the standard normal distribution.

The behaviour of a neuron using this approach can be interpreted as follows: As the (deterministic) membrane potential increases, the probability of the neuron firing increases. Additionally, as the amount of presynaptic activity increases, the more "random" a neurons output becomes (i. e. the probability of firing moves towards 0.5).

\section*{E - Additional Experiments, Results, and Training Details}

\subsection*{Additional Experiments}
Unless otherwise mentioned, the additional experiments performed here are conducted using a binary network on CIFAR-10. 

\subsubsection*{IW-ST experiments}
Here we investigate the practical performance of the IW-ST($p$) estimator for different choices of $p$. We compare four choices of $p$: i) $p=0$, LHS Euler rule (P0 in Fig. \ref{fig:iwst}), ii) $p=1$, RHS Euler rule (P1 in Fig. \ref{fig:iwst}), iii) $p=0.5$, symmetric trapezoid rule for low bias (LB in Fig. \ref{fig:iwst}), iv) $p=F(h)$, classical ST, v) $p=1$ if $F(h)>0.5$ and $p=0$ if $F(h)<0.5$, (low variance LV in Fig. \ref{fig:iwst}). The low variance approach is based on the end point (0 or 1) of the low-variance interval derived in propositions (3-4). 

Fig. \ref{fig:iwst} shows training loss and test accuracy for a representative run, with similar trends observed across multiple seeds. Both Euler-based methods underperform, consistent with the theoretical findings that fixed endpoints do not lie within the low-bias or low-variance intervals (Propositions 3–4). The RHS method ($p=1$) performs especially poorly, likely due to propagating gradients only through firing neurons. In sparse activity regimes, this limits learning signals. Conversely, $p=0$ propagates gradients through silent neurons, which may constitute the majority and thus offer more informative gradients.

These findings, although theoretically anticipated, challenge prevailing assumptions in biologically inspired learning rules such as Hebbian plasticity and spike-timing-dependent plasticity (STDP), where updates occur only between co-active neurons. Moreover, alternative gradient estimators for SNNs based on differentiating through the precise spike time, such as in \cite{wunderlich2021event},  also only propagate gradients through active neurons, though their empirical performance often lags. Our derivation offers an explanation: ST-based estimators implicitly incorporate noise and enable credit assignment to all neurons near firing threshold—yielding lower variance gradients and more robust learning signals.

Among the non-Euler estimators, performance differences are minimal. The ST estimator exhibits a slight edge, making it a pragmatic choice given its simplicity and robustness—it effectively interpolates between the low-bias and low-variance approaches.

\begin{figure}[tb]
    \centering
    % first subplot
    \begin{minipage}[b]{0.48\linewidth}
        \centering
        \includegraphics[width=\linewidth]{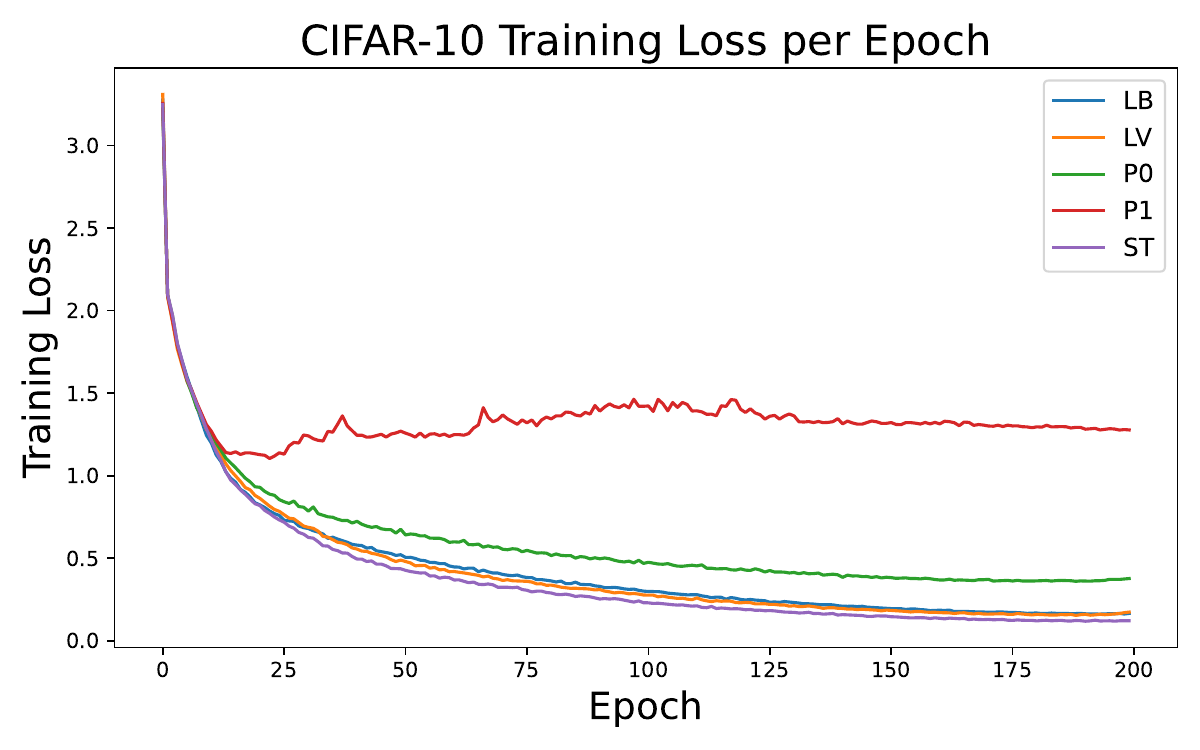}

    \end{minipage}
    \hfill
    % second subplot
    \begin{minipage}[b]{0.48\linewidth}
        \centering
        \includegraphics[width=\linewidth]{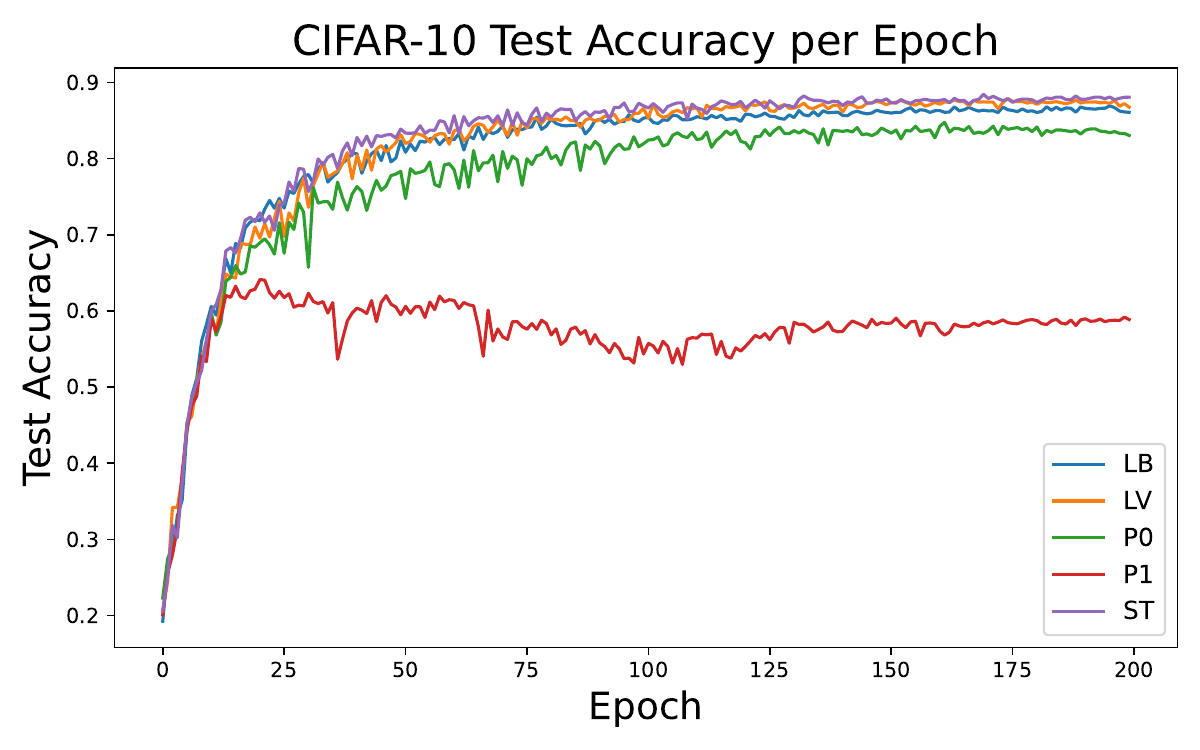}
    \end{minipage}
    
    \caption{Training loss and test accuracy  for IW-ST($p$) estimators on CIFAR-10}
    \label{fig:iwst}
\end{figure}

\subsubsection*{AGR estimators}

Here we assess the performance of AGR estimators using sigmoid relaxations with temperatures: $k=1$ and $k=0.2$, and compare them against the ST estimator. Note that the only critical difference between AGR estimators and our IW-ST($p$) estimators is the additional dampening performed by the former. 

Fig. \ref{fig:agr} shows that the AGR estimator significantly underperforms compared to the ST estimator, regardless of the temperature. In deep binary networks, cumulative damping hinders effective gradient propagation, particularly in deeper layers, leading to reduced training efficiency.

\begin{figure}[tb]
    \centering
    % first subplot
    \begin{minipage}[b]{0.48\linewidth}
        \centering
        \includegraphics[width=\linewidth]{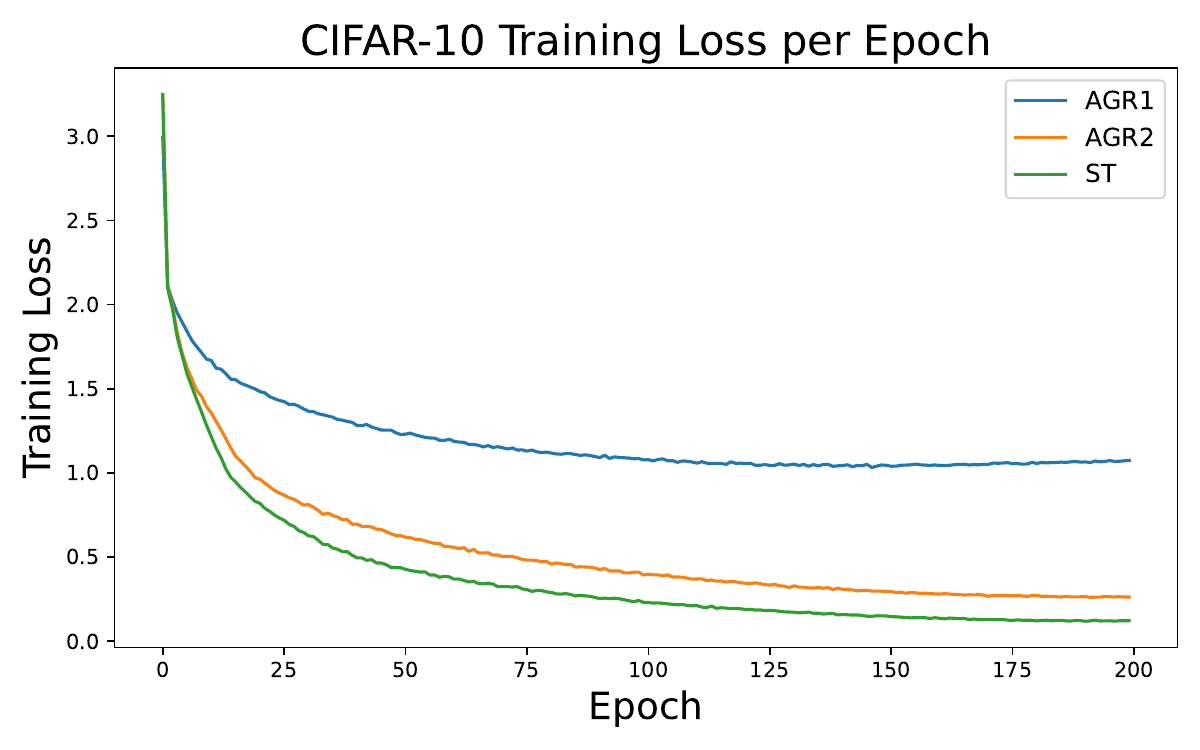}

    \end{minipage}
    \hfill
    % second subplot
    \begin{minipage}[b]{0.48\linewidth}
        \centering
        \includegraphics[width=\linewidth]{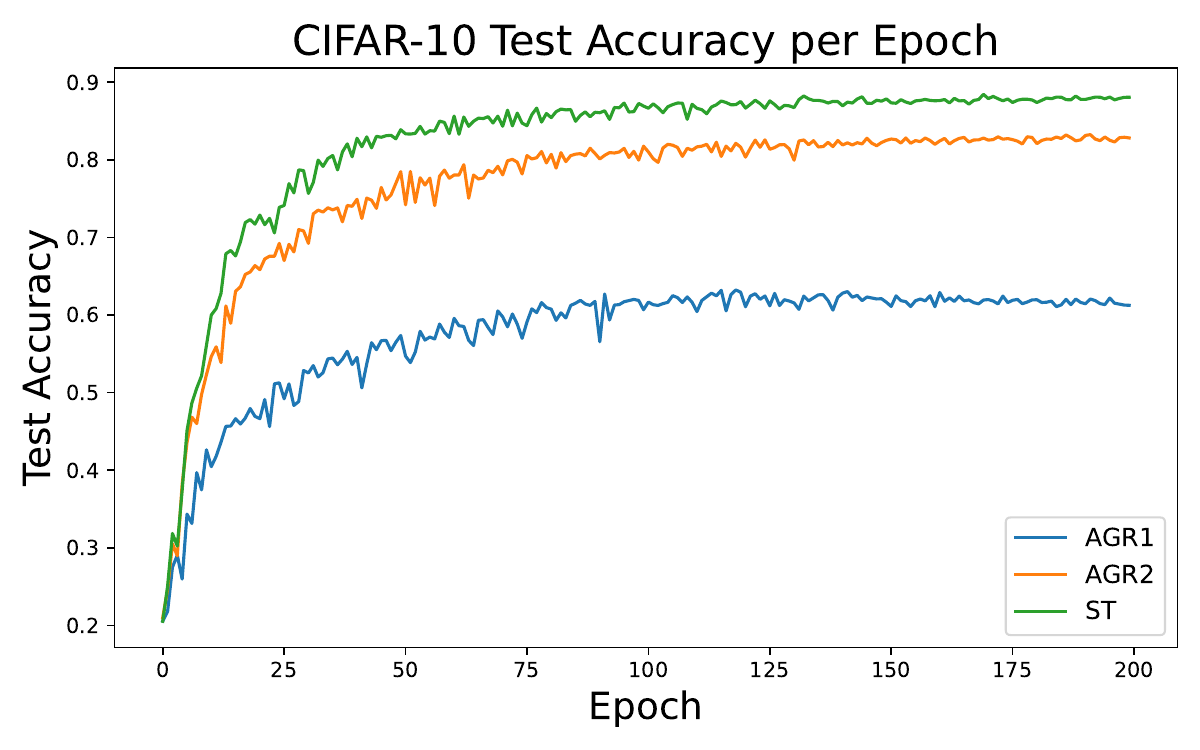}
 
    \end{minipage}
    \caption{Training loss and test accuracy for AGR estimators on CIFAR-10}
    \label{fig:agr}
\end{figure}

\subsubsection*{Vanishing gradients}

To test our hypothesis that the KL divergence term mitigates vanishing gradients, we analysed per-layer the gradient norms after 30 epochs, comparing our models with a range of KL weights ($\lambda$) and batch normalisation. 

As shown in Fig. \ref{fig:vangrad}, models with KL regularisation or batch normalisation maintain consistent gradient norms across layers.  Without the KL term, gradients vanish in deeper layers. This effect is robust across a wide range of $\lambda$ values. The attenuation factor plot, Fig. \ref{fig:attenutation}, further support our hypothesis: KL-regularised models exhibit stable attenuation factors across epochs, closely matching those of batch-normalised networks. 

\begin{figure}[tb]
    \centering
    % first subplot
    \begin{minipage}[b]{0.48\linewidth}
        \centering
        \includegraphics[width=\linewidth]{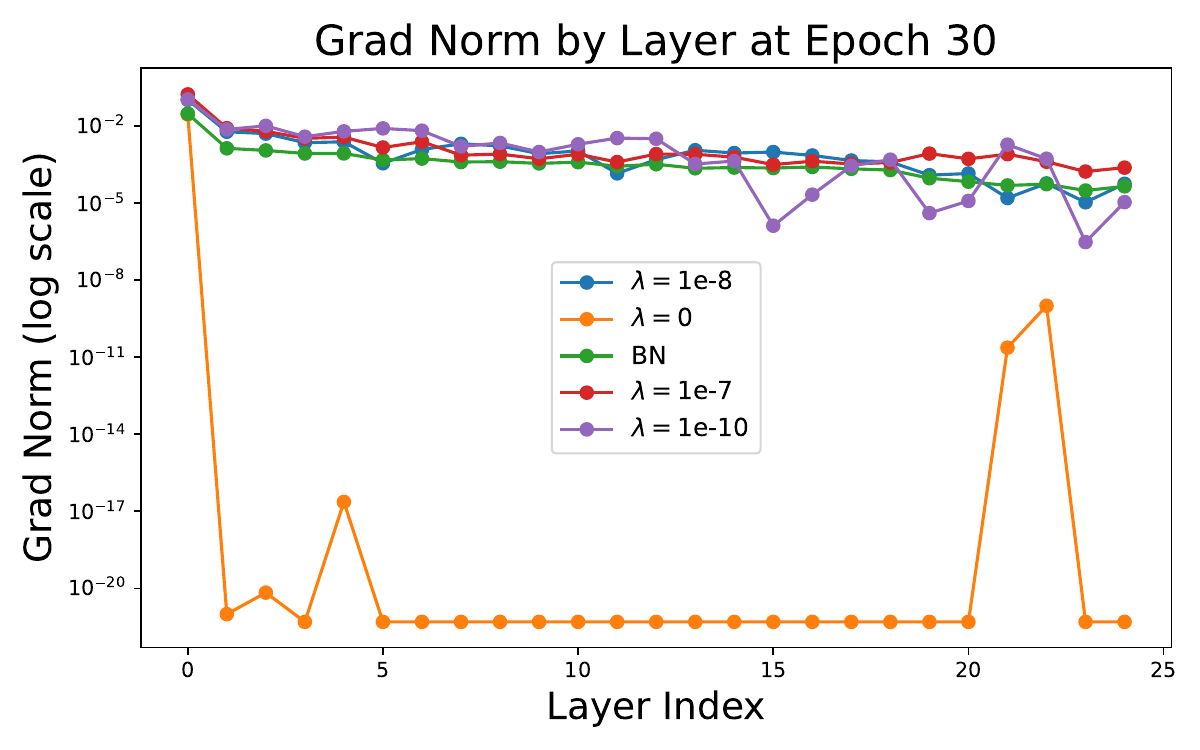}
        \caption{Gradient norms per layer after 30 epochs}
        \label{fig:vangrad}
    \end{minipage}
    \hfill
    % second subplot
    \begin{minipage}[b]{0.48\linewidth}
        \centering
        \includegraphics[width=\linewidth]{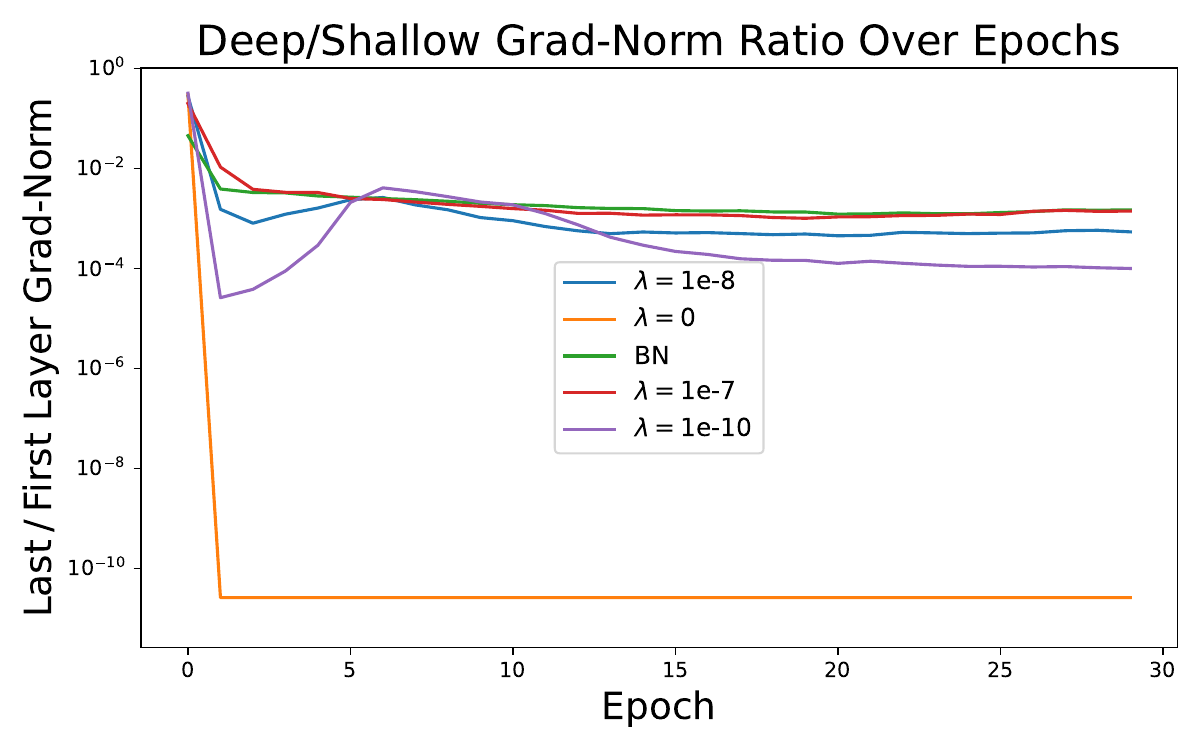}
        \caption{Attenuation factor over first 30 Epochs}
        \label{fig:attenutation}
    \end{minipage}
\end{figure}

\subsubsection*{Learnable scale parameters}

In the main paper, we considered one variance parameter per layer. Here we consider both one variance parameter per neuron, and one variance parameter per weight. The advantage of these additional parameters is that we enable the model to reflect uncertainty at a more granular level, hence are more expressive. On the other hand, the more parameters we use increases the complexity and computational overhead of the approach. 

In Fig.~\ref{fig:scales}, we plot the training loss and test accuracy across epochs for three options: one variance per layer, one variance per neuron, and one variance per weight. As we can see, these methods perform very similarly in terms of training efficiency and generalisation, suggesting that unless the advantages in terms of uncertainty quantification are desired, one variance parameter per layer suffices.

\begin{figure}[tb]
    \centering
    % first subplot
    \begin{minipage}[b]{0.48\linewidth}
        \centering
        \includegraphics[width=\linewidth]{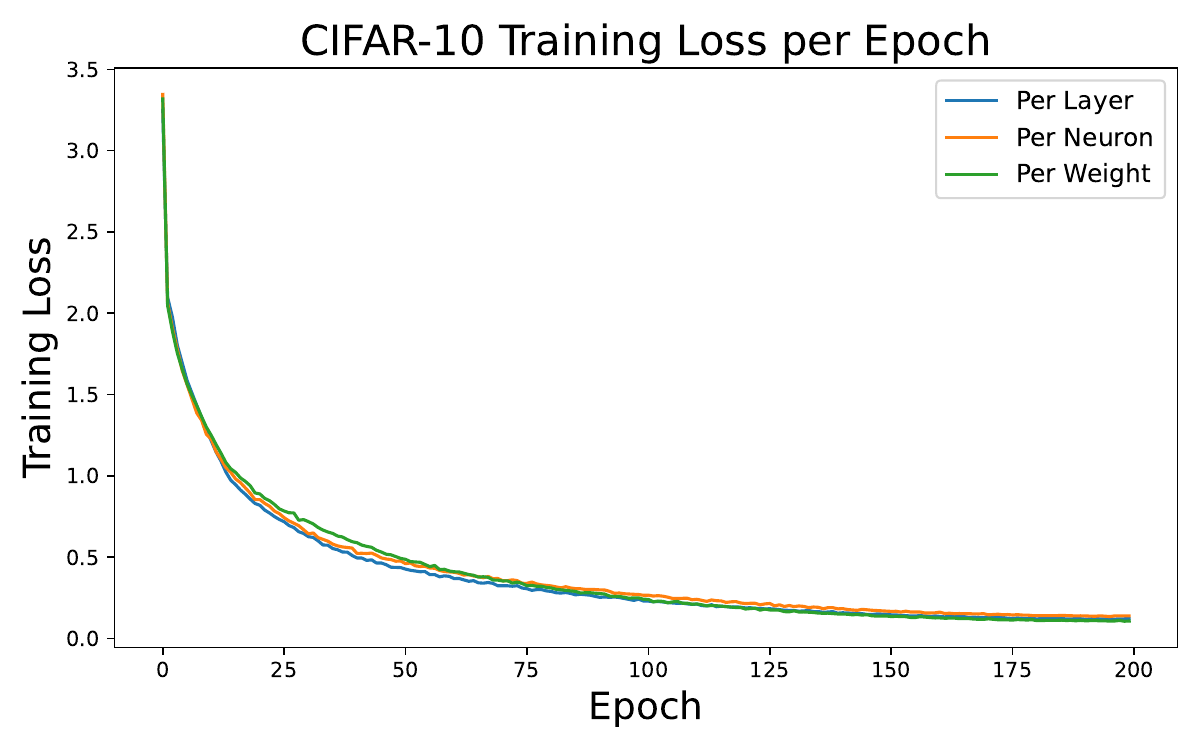}
    
    \end{minipage}
    \hfill
    % second subplot
    \begin{minipage}[b]{0.48\linewidth}
        \centering
        \includegraphics[width=\linewidth]{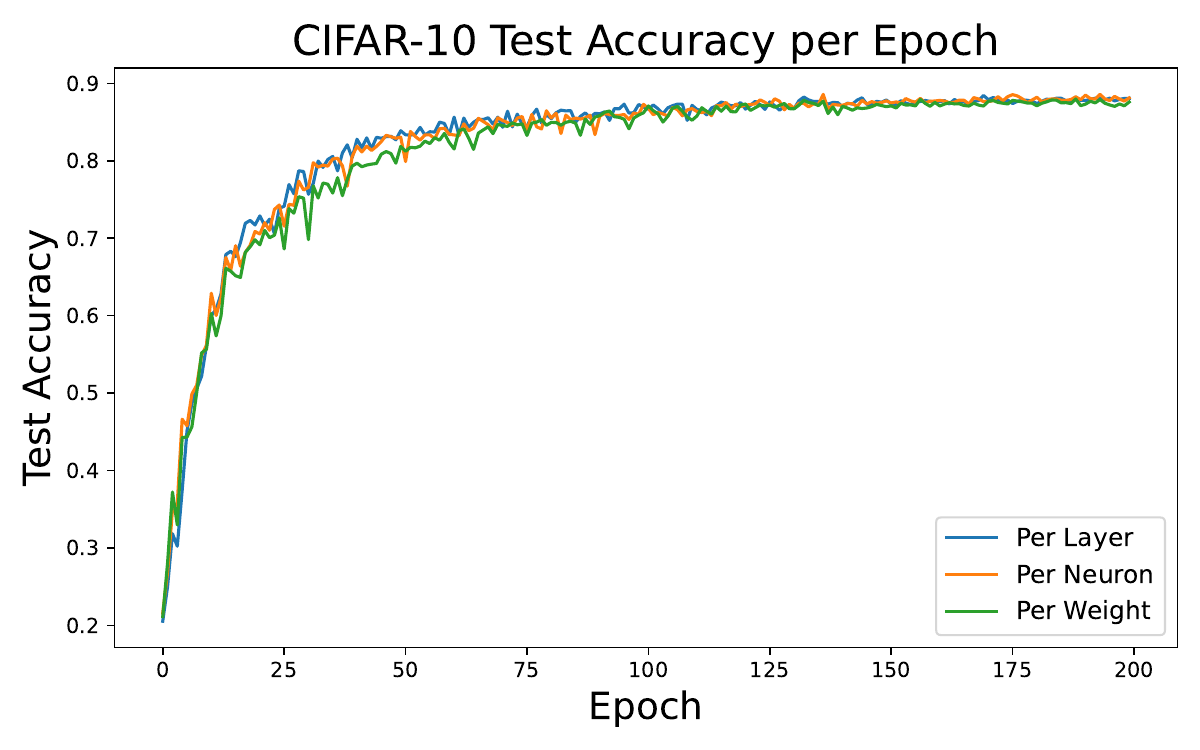}

    \end{minipage}
    
    \caption{Training loss and test accuracy for different parameter sharing settings on CIFAR-10}
    \label{fig:scales}
\end{figure}

\subsubsection*{Results tables}
Table \ref{results} reports the mean test accuracy $\pm$ one standard deviation based on 5 training runs for the experiments discussed in the main text to ensure consistency across random seeds. For all experiments with single-run results, we verified that observed patterns generalize across seeds.

\begin{table}[tb]
  \centering
  \footnotesize
  %─── First mini‐table ──────────────────────────────────────────────────────
  \begin{minipage}[t]{0.32\textwidth}
    \centering
    \textbf{CIFAR-10}\\[1ex]
    \begin{tabular}{lc}
      \hline
      Training Variant& Accuracy   \\
      \hline
      BBNN           & 88.0$\pm$0.2\\
      BBNN-NKL       & 45.3$\pm$0.5\\
      BBNN-MFA       & 88.3$\pm$0.2\\
      BBNN-FPV       & 90.3$\pm$0.3\\
      SG-BN          & 91.1$\pm$0.6\\
      \hline
    \end{tabular}
  \end{minipage}%
  \hfill
  %─── Second mini‐table ─────────────────────────────────────────────────────
  \begin{minipage}[t]{0.32\textwidth}
    \centering
    \textbf{DVS Gesture}\\[1ex]
    \begin{tabular}{lc}
      \hline
      Training Variant& Accuracy   \\
      \hline
      SBNN           & 93.5$\pm$ 1\\
      SBNN-NKL       & 75.3$\pm$1.8\\
      SG-BN          & 94$\pm$0.8\\
      SG-NBN-1       & 90.3$\pm$1.3\\
      SG-NBN-2       & 78.1$\pm$1.4\\
      \hline
    \end{tabular}
  \end{minipage}%
  \hfill
  %─── Third mini‐table ──────────────────────────────────────────────────────
  \begin{minipage}[t]{0.32\textwidth}
    \centering
    \textbf{SHD}\\[1ex]
    \begin{tabular}{lc}
      \hline
      Training Variant& Accuracy   \\
      \hline
      SBNN           & 85.1$\pm$1.3\\
      SBNN-NKL       & 70.3$\pm$1.1\\
      SG-BN          & 84.3$\pm$1.3\\
      SG-NBN-1       & 72.1$\pm$1.6\\
      SG-NBN-2       & 67.1$\pm$1.8\\
      \hline
    \end{tabular}
  \end{minipage}

  \vspace{1ex}
  \caption{Mean accuracy ($\pm$ std) for each dataset.}
  \label{results}
\end{table}

\subsection*{Training details}

\textbf{CIFAR-10}: \textit{Architecture}: 26-layer Binary ResNet (64→128→256→512 channels).

\textit{Data pre-processing}: Normalised inputs and standard data augmentation (Random Crop+Flip)
\begin{table}[tb]
  \centering
  \begin{tabular}{ll}
    \toprule
    \textbf{Hyperparameter}          & \textbf{Value}                             \\
    \midrule
    Training epochs                  & 200\\
    Batch size                       & 256                                        \\
    Optimizer                        & Adam                                       \\
    Learning rates                   & \(\ell_{\text{weights}} = 0.005\), \(\ell_{\text{scales}} = 0.05\)\\
    Learning‐rate schedule           & Cosine annealing over 200 epochs (min lr=initial/50)\\
    KL weight \(\lambda\)& \(1\times10^{-6}\)\\
    Neuron threshold& 0\\
    Surrogate function& Normal PDF with $\sigma=0.4$\\
    \bottomrule
  \end{tabular}
  \caption{CIFAR-10 training hyperparameters.}
  \label{tab:cifar10-train}
\end{table}

\textbf{DVS-128 Gesture}: \textit{Architecture}: 20-layer Spiking ResNet (32→64→128 channels).

\textit{Data pre-processing}: Raw event streams are binned into $T{=}50$ bins, yielding $50 {\times} 2 {\times}128 \times 128$ tensors. The $128 \times 128$ spatial dimensions are then bilinearly down-sampled to $32{\times}32$. 

\begin{table}[tb]
  \centering
  \begin{tabular}{ll}
    \toprule
    \textbf{Hyperparameter}          & \textbf{Value}                             \\
    \midrule
    Training epochs                  & 70                                         \\
    Batch size                       & 64                                         \\
    Optimizer                        & Adam                                       \\
    Initial learning rates           & \(\ell_{1}=0.005\), \(\ell_{2}=0.05\)      \\
    Learning‐rate schedule           & Cosine annealing over 70 epochs (min lr=initial/50)\\
    KL weight \(\lambda\)& \(1\times10^{-10}\)                        \\
    Time steps& 50\\
 Neuron threshold&1\\
 Surrogate function&Normal PDF with $\sigma=0.4$\\
 \bottomrule
  \end{tabular}
  \caption{DVS-Gesture training hyperparameters.}
  \label{tab:dvs-train}
\end{table}

\textbf{SHD (Spiking Heidelberg Digits)}: \textit{Architecture}: 20-layer Spiking ResNet (32→64→128 channels).

\textit{Data pre-processing}: Raw event streams are binned into $T{=}100$ bins, yielding $100{\times}700$ tensors.

\begin{table}[tb]
  \centering
  \begin{tabular}{ll}
    \toprule
    \textbf{Hyperparameter}          & \textbf{Value}                             \\
    \midrule
    Training epochs                  & 30                                         \\
    Batch size                       & 32                                         \\
    Optimizer                        & Adam                                       \\
    Initial learning rates           & \(\ell_{\text{weights}} = 0.005\), \(\ell_{\text{scales}} = 0.05\)\\
    Learning‐rate schedule           & Cosine annealing over 30 epochs (min lr=initial/10)\\
    KL weight \(\lambda\)& \(1\times10^{-6}\)                         \\
 Time steps&100\\
    Neuron threshold& 1
\\
    Surrogate function& Normal PDF with $\sigma=0.4$\\
    \bottomrule
  \end{tabular}
  \caption{SHD training hyperparameters.}
  \label{tab:shd-train}
\end{table}

\textbf{Additional details (all datasets)}
\begin{itemize}
\item    For parameter initialisation, we used Kaiming-uniform for convolution and fully connected weights. Scale parameters were initialised at $\sigma_0\!=\!0.5/\!\sqrt{\text{fan-in}}$ 
\item We report the model performance from the \emph{final} epoch in every case; no early stopping was applied.
\item All experiments were run on a single NVIDIA\,H100 GPU. 
\item   For models without batch normalisation, we use a learnable per-channel scale and bias.
\item  Hyperparameters selected by a simple sweep. 
\item For SNNs, we include an additional 'base' level of noise per time step as we found this improves performance. This noise can be interpreted in the Bayesian context as arising from uncertainty in the neuron threshold. 
\item We use the cross-entropy loss for all tasks. For spiking networks, the logits are calculated as the sum of membrane potentials in a read-out layer across all time steps.
\end{itemize}

\end{document}